\documentclass[11pt]{article}

\usepackage[margin=1.02in]{geometry}
\usepackage{amsmath,amssymb,amsfonts,amsthm,mathtools}
\usepackage{algorithm}
\usepackage[noend]{algpseudocode}
\usepackage{booktabs}
\usepackage{enumitem}
\usepackage{microtype}
\usepackage[authoryear,round,sort]{natbib}
\usepackage[T1]{fontenc}
\usepackage{lmodern}
\usepackage{xcolor}
\usepackage[colorlinks=true,linkcolor=blue!55!black,citecolor=blue!55!black,urlcolor=blue!55!black]{hyperref}
\usepackage{aliascnt}
\usepackage[nameinlink,capitalize]{cleveref}

\allowdisplaybreaks
\newtheorem{theorem}{Theorem}[section]
\newaliascnt{lemma}{theorem}
\newtheorem{lemma}[lemma]{Lemma}
\aliascntresetthe{lemma}
\newaliascnt{proposition}{theorem}
\newtheorem{proposition}[proposition]{Proposition}
\aliascntresetthe{proposition}
\newaliascnt{corollary}{theorem}
\newtheorem{corollary}[corollary]{Corollary}
\aliascntresetthe{corollary}
\newaliascnt{assumption}{theorem}

\aliascntresetthe{assumption}
\theoremstyle{definition}
\newaliascnt{definition}{theorem}
\newtheorem{definition}[definition]{Definition}
\aliascntresetthe{definition}
\newtheorem{question}{Question}
\newaliascnt{example}{theorem}
\newtheorem{example}[example]{Example}
\aliascntresetthe{example}
\theoremstyle{remark}
\newaliascnt{remark}{theorem}
\newtheorem{remark}[remark]{Remark}
\aliascntresetthe{remark}
\newaliascnt{openproblem}{theorem}

\aliascntresetthe{openproblem}

\crefname{theorem}{theorem}{theorems}
\Crefname{theorem}{Theorem}{Theorems}
\crefname{lemma}{lemma}{lemmas}
\Crefname{lemma}{Lemma}{Lemmas}
\crefname{proposition}{proposition}{propositions}
\Crefname{proposition}{Proposition}{Propositions}
\crefname{corollary}{corollary}{corollaries}
\Crefname{corollary}{Corollary}{Corollaries}
\crefname{assumption}{assumption}{assumptions}
\Crefname{assumption}{Assumption}{Assumptions}
\crefname{definition}{definition}{definitions}
\Crefname{definition}{Definition}{Definitions}
\crefname{question}{question}{questions}
\Crefname{question}{Question}{Questions}
\crefname{example}{example}{examples}
\Crefname{example}{Example}{Examples}
\crefname{remark}{remark}{remarks}
\Crefname{remark}{Remark}{Remarks}
\crefname{openproblem}{open problem}{open problems}
\Crefname{openproblem}{Open Problem}{Open Problems}

\newcommand{\E}{\mathbb E}

\newcommand{\R}{\mathbb R}
\newcommand{\1}{\mathbf 1}
\newcommand{\cA}{\mathcal A}
\newcommand{\cD}{\mathcal D}

\newcommand{\cG}{\mathcal G}
\newcommand{\cH}{\mathcal H}
\newcommand{\cI}{\mathcal I}
\newcommand{\cL}{\mathcal L}
\newcommand{\cM}{\mathcal M}
\newcommand{\cP}{\mathcal P}
\newcommand{\cR}{\mathcal R}
\newcommand{\cS}{\mathcal S}
\newcommand{\cV}{\mathcal V}
\newcommand{\cX}{\mathcal X}
\newcommand{\Ber}{\operatorname{Ber}}
\newcommand{\gam}{\gamma}

\newcommand{\TV}{\operatorname{TV}}

\newcommand{\Bias}{\operatorname{Bias}}
\newcommand{\Mass}{\operatorname{Mass}}
\newcommand{\SwapReg}{\operatorname{SwapReg}}
\newcommand{\SwapAgn}{\operatorname{SwapAgn}}
\newcommand{\Th}{\operatorname{Th}}

\newcommand{\argmin}{\operatorname*{arg\,min}}
\newcommand{\argmax}{\operatorname*{arg\,max}}

\newcommand{\abs}[1]{\left|#1\right|}

\newcommand{\eps}{\varepsilon}

\title{Fast Rates for Swap-Agnostic Learning of Proper Losses}
\author{Princewill Okoroafor\\
Harvard University\\
\texttt{princewill\_okoroafor@seas.harvard.edu}}
\date{}

\begin{document}
\maketitle

\begin{abstract}
Swap-agnostic learning strengthens classical agnostic learning by allowing the
comparator to use a different hypothesis on each prediction level set.  This
benchmark captures prediction-dependent postprocessing.  However, allowing the
comparator to depend on the prediction value seems to require solving a
separate agnostic-learning problem at each possible prediction value.  We show
that, for proper losses, these prediction-level comparisons can be controlled
jointly.

Our main result is an offline swap-agnostic learner for any fixed proper loss.
For a finite hypothesis class $\cH$, the excess risk for a fixed smooth proper
loss scales as
$\widetilde O((\log(|\cH|)/m)^{2/3})$ from $m$ i.i.d.\ samples, with the
corresponding online swap regret scaling as
$\widetilde O(T^{1/3}(\log|\cH|)^{2/3})$.

We also give algorithms whose predictions are swap-agnostic simultaneously for
entire families of losses.  For all proper losses bounded in $[-1,1]$, the
rates are $\widetilde O(\sqrt{T\log|\cH|})$ online and
$\widetilde O(\sqrt{\log|\cH|/m})$ offline.  For convex $1$-Lipschitz proper
losses, the rates improve to
$\widetilde O(T^{1/3}(\log|\cH|)^{2/3})$ online and
$\widetilde O((\log|\cH|/m)^{2/3})$ offline.  These rates are tight up to
logarithmic factors and improve on the
$\widetilde O(T^{2/3}(\log|\cH|)^{1/3})$ rates implied by the
swap-omniprediction guarantee of
\citet{luo2025improved}.  Our main technical approach reduces swap-agnostic
learning to a second-order form of multicalibration, which we obtain using
Blackwell approachability with a Bernstein-style variance correction.
\end{abstract}

\tableofcontents

\section{Introduction}

Classical agnostic learning asks a learner to compete with the best fixed
hypothesis in a reference class $\cH$, without assuming realizability
\citep{vapnik1998statistical,shalev2014understanding}.  In the offline
setting, the learner receives i.i.d.\ samples from an unknown distribution
$\cD$ over $(X,Y)$ and outputs a predictor; for a loss $\ell(p,y)$, the
benchmark asks whether the predictor performs nearly as well as the best
single hypothesis $h\in\cH$.

In many applications, however, the prediction is not the final decision.  A
probability forecast may be thresholded, discretized, or passed to a
cost-sensitive decision rule, and changing the downstream costs changes the
optimal threshold and therefore the decision rule that gets applied.  A
clinical risk score triggers different follow-up protocols at low, medium,
and high risk; a credit score is routed to different underwriting policies
across score bands; an ad-ranking score is postprocessed differently for
cold-start items than for high-confidence ones
\citep{grundy2019cholesterol,cfpb2024riskpricing,ardywibowo2025bayescns}.  In
each case the predicted probability is an input to a later decision rather
than the action itself -- the standard role of probabilistic predictions in
cost-sensitive classification and surrogate-loss theory
\citep{elkan2001foundations,bartlett2006convex}, and central to the
decision-theoretic interpretation of proper scoring rules
\citep{dawid1982well,gneiting2007strictly}.

Swap-agnostic learning strengthens the comparator to reflect exactly this
freedom \citep{gopalan2023swap}.  Rather than asking whether one fixed
hypothesis beats the predictor globally, it asks whether the predictor can
be improved after first looking at its own predicted value.  If predictions
lie on the grid $\Gamma_N=\{\gam_i=i/N:0\le i\le N\}$, a hypothesis swap is a
map $\phi:\Gamma_N\to\cH$, and when the learner predicts $p$ the comparator
answers with $\phi(p)$ instead of one fixed hypothesis -- so a score near
$0.1$ and a score near $0.9$ may be judged against entirely different
hypotheses.  For a randomized offline predictor $\hat p$, this gives the
excess-risk benchmark
\[
  \SwapAgn_\cD(\ell,\cH)
  =\sup_{\phi:\Gamma_N\to\cH}
  \E\!\left[\ell(p,Y)-\ell(\phi(p)(X),Y)\right],
\]
with $(X,Y)\sim\cD$ and $p\sim\hat p(X)$, and online it becomes swap regret,
\[
  \SwapReg_T(\ell,\cH)
  =\sup_{\phi:\Gamma_N\to\cH}
  \sum_{t=1}^T
  \left[\ell(p_t,y_t)-\ell(\phi(p_t)(x_t),y_t)\right].
\]
Ordinary agnostic learning is recovered when the same hypothesis is used at
every prediction level, so swap-agnostic learning is never easier than the
ordinary benchmark -- the question is how much harder it really is.

A natural first approach already suggests a worrying answer.  A standard
multicalibration-based algorithm can be read as solving a separate
agnostic-learning problem inside each of the $N+1$ prediction buckets; if
bucket $i$ collects $n_i$ examples, summing the usual $\sqrt{n_i}$-type rates
over all of them produces an apparent $\sqrt N$ overhead, and the known
route to swap-agnostic learning through multicalibration does indeed yield
$\widetilde O(T^{2/3}(\log|\cH|)^{1/3})$ swap regret
\citep{hebertjohnson2018multicalibration,gopalan2023swap,luo2025improved}.
This bucket-by-bucket analysis does not tell us whether the extra cost is
real, or just an artifact of treating every bucket as its own independent
statistical problem.

\begin{question}[Single-loss swap-agnostic learning]
Can offline swap-agnostic learning for a fixed proper loss avoid paying a
separate statistical price for every prediction bucket?
\end{question}

A second question arises once the downstream objective is allowed to move.
Omniprediction asks for one predictor that supports near-optimal decisions
for many losses at once \citep{gopalan2022omnipredictors}; in the swap
setting, swap omniprediction asks the same predictor to be swap-agnostic for
every loss in a family,\footnote{Swap omniprediction lets the comparator
choose a different loss, together with a different hypothesis, at every
prediction level, so the loss itself is swapped along with the hypothesis.
Question~2 instead fixes one loss for the entire run and asks a single
predictor to satisfy the ordinary swap-agnostic guarantee for every such
loss, one at a time; the comparator never swaps between losses within a
run.  \Cref{sec:intro-related-work} makes this ordering of benchmarks
precise.} and the existing route to that guarantee runs through
multicalibration as well \citep{gopalan2023swap}, inheriting the same
quantitative overhead.

\begin{question}[Simultaneous swap-agnostic learning for loss families]
Can one offline predictor be swap-agnostic simultaneously for broad families
of proper losses, such as bounded proper losses or convex $1$-Lipschitz
proper losses?
\end{question}

We answer both questions affirmatively, and the key to both is the same:
stop treating prediction buckets as independent statistical problems, and
instead measure directly how favorable each bucket already is to the
predictor before paying anything to correct it.  A single algorithmic
guarantee -- a second-order strengthening of multicalibration developed in
\Cref{sec:somc} -- turns this favorable measurement into a bound that
absorbs the bucket cost instead of paying it once per bucket.  The rest of
this section describes what this buys us and how it works.

\subsection{Contributions}

\paragraph{A single proper loss.}
For any proper loss $\ell$ with $L$-Lipschitz partial losses and finite
hypothesis class $\cH$, one algorithm achieves
\[
  \SwapReg_T(\ell,\cH)=\widetilde O\!\big(LT^{1/3}(\log|\cH|)^{2/3}\big)
  \ \text{online,}
  \qquad
  \SwapAgn_\cD(\ell,\cH)=\widetilde O\!\Big(L\big(\tfrac{\log|\cH|}m\big)^{2/3}\Big)
  \ \text{offline,}
\]
with no convexity assumption needed (\Cref{thm:single-lipschitz}).  Both
exponents are tight, matching an $\Omega(T^{1/3})$ lower bound for the
half-Brier loss against a fixed three-hypothesis class
(\Cref{thm:brier-lb} and \Cref{cor:optimal-smooth-time}).

\paragraph{Families of losses.}
The same guarantee extends at only a logarithmic cost.  A finite family
$\mathcal L$ sharing a Lipschitz constant costs an extra $\log|\mathcal L|$
(\Cref{thm:finite-lipschitz-family}).  Every bounded proper loss decomposes
into a positive mixture of V-shaped basis losses
\citep{kleinberg2023ucalibration,li2022optimization}, but these are
nonsmooth, so controlling them needs a localized version of the algorithm
together with a new argument that extends the guarantee from finitely many
thresholds to every threshold at once (\Cref{sec:bounded}).  This gives
\emph{every} proper loss bounded in $[-1,1]$ at once the coarser rate
\[
  \SwapReg_T(\ell,\cH)=\widetilde O\!\big(\sqrt{T\log|\cH|}\big)
  \ \text{online,}
  \qquad
  \SwapAgn_\cD(\ell,\cH)=\widetilde O\!\big(\sqrt{\log|\cH|/m}\big)
  \ \text{offline,}
\]
uniformly over $\ell$ (\Cref{thm:bounded-proper-rates}), matching known
lower bounds (\Cref{lem:v-online-lb,lem:v-offline-lb}).  This extends the
payoff-bounded downstream-task guarantee of \citet{hu2024payoffbounded} to
hypothesis-valued comparators.

\paragraph{Convex Lipschitz losses.}
We prove a new positive decomposition of every convex $1$-Lipschitz proper
loss into clipped-ReLU basis losses (\Cref{prop:clipped-decomposition-front}),
which recovers the sharper single-loss exponents for the entire family at
once,
\[
  \SwapReg_T(\ell,\cH)=\widetilde O\!\big(T^{1/3}(\log|\cH|)^{2/3}\big)
  \ \text{online,}
  \qquad
  \SwapAgn_\cD(\ell,\cH)=\widetilde O\!\Big(\big(\tfrac{\log|\cH|}m\big)^{2/3}\Big)
  \ \text{offline,}
\]
uniformly over $\ell$ (\Cref{thm:offline-convex-lipschitz}).

\paragraph{Beyond proper losses.}
Canonical properization (\Cref{lem:canonical-properization}) transfers every
rate above, with no loss in the exponent, to an arbitrary bounded binary
loss evaluated after best-response postprocessing: the relevant comparator
complexity becomes the loss's own slope class rather than $\cH$ directly
(\Cref{sec:improper}).

\paragraph{The underlying engine.}
All of this follows from one online guarantee, second-order
multicalibration.  For a transcript $(x_t,p_t,y_t)$ and a test
$g:\cX\times\Gamma_N\to[-1,1]$, write
\[
  \Bias_T(g)=\sum_{t=1}^T(y_t-p_t)g(x_t,p_t),
  \qquad
  \Mass_T(g)=\sum_{t=1}^Tg(x_t,p_t)^2.
\]
Our algorithm guarantees, simultaneously for every test in a fixed finite
class and with high probability,
\[
  \Bias_T(g)\le c\,\Mass_T(g)+a,
\]
for a learning-rate-dependent constant $c$ and an additive term $a$ scaling
with the log-size of the test class (\Cref{def:somc}, achieved with high
probability by \Cref{thm:online-somc}).  An online-to-batch conversion
(\Cref{thm:online-to-batch}) transfers this bound offline; every rate above
comes from tuning $c$ to the loss or family at hand, exactly as the
technical overview below explains.

\subsection{Technical overview}

Fix one proper loss $\ell$.  When a prediction $p$ is compared against an
alternative prediction $q$, the proper-loss identity of
\Cref{lem:proper-identity} splits the comparison into two pieces:
\begin{equation}
  \ell(p,y)-\ell(q,y)
  =\underbrace{(y-p)(\Delta_\ell(p)-\Delta_\ell(q))}_{\text{bias term}}
   -\underbrace{[L_p(q)-L_p(p)]}_{\text{drift term}\ \ge0}.
  \label{eq:intro-identity}
\end{equation}
The bias term is proportional to the calibration residual $y-p$; it is what
ordinary multicalibration controls, and controlling it directly bucket by
bucket is exactly the expense identified above.  The drift term is always
nonnegative -- moving the prediction away from a bucket's best response is
only ever penalized, never rewarded, once the truth is revealed -- and,
crucially, it grows with precisely the quantity that makes the bias term
large in the first place, $\Delta_\ell(p)-\Delta_\ell(q)$.  A
quadratic-growth inequality (\Cref{lem:quadratic-growth}) makes this
precise: $(\Delta_\ell(p)-\Delta_\ell(q))^2\lesssim K\cdot[L_p(q)-L_p(p)]$
for a constant $K$ governed by the loss's Lipschitz or smoothness constant.
So the same quantity that makes a bucket's comparison risky also makes that
bucket's own drift large: the drift is collateral the loss's geometry
already offers, and the only remaining question is how to collect on it.

This is exactly what second-order multicalibration is for.  The online
algorithm's guarantee bounds the bias term, $\Bias_T(g)$, by a constant
multiple of $\Mass_T(g)$ plus an additive term, where the constant is set by
the learning rate.  Tuning the learning rate so that this constant matches
$K$ exactly (\Cref{thm:single-lipschitz}) makes the quadratic bias and drift
terms cancel identically, leaving only the additive term, which scales with
$\log|\cH|$ rather than with the number of prediction buckets.  Balancing
the grid resolution $N$ against this cancellation -- a finer grid shrinks
discretization error but creates more tests to control -- then gives the
$T^{1/3}$ online and $m^{-2/3}$ offline exponents.  This is precisely where
the improvement over the bucket-wise route occurs: that route pays an
ordinary agnostic-learning regret independently inside each bucket, so
refining the grid also multiplies the number of independently-paid-for
buckets, whereas here the offset lets the per-bucket cost enter only
additively, through $\log|\cH|$ -- which is why the algorithm can afford a
finer grid before discretization error dominates, landing at $T^{1/3}$
instead of $T^{2/3}$.

The reduction is modular in a way that pays off immediately for families of
losses: once a loss has told the algorithm which selector tests to control,
the algorithm itself does not care which loss those tests came from.  A
finite family $\mathcal L$ is therefore handled by simply pooling every
loss's tests into one larger test class, and since the guarantee depends
only logarithmically on test-class size, pooling in $|\mathcal L|$ losses'
worth of tests adds only $\log|\mathcal L|$ to the bound rather than
multiplying it.  Infinite families need the further idea already mentioned
above -- writing a complicated loss as a positive mixture of simpler basis
losses, with prediction-independent affine terms canceling against every
comparator -- and for bounded proper losses the relevant basis is the
V-shaped threshold decomposition of
\citet{kleinberg2023ucalibration,li2022optimization}
(\Cref{lem:proper-measure-representation}).  These basis losses are
nonsmooth -- their slope jumps discontinuously at the threshold -- so the
single-loss cancellation argument above (\Cref{lem:quadratic-growth}) does
not apply to them directly.  Instead we run a \emph{localized} version of
the algorithm, with one test per prediction bucket that is active only on
the rounds where the prediction equals that bucket's grid value, since the
favorable drift near a V-shaped threshold depends on the distance from the
current prediction to the threshold (\Cref{lem:v-drift}), which differs
bucket by bucket rather than being governed by one global constant.
Second-order control on every bucket's test at once, followed by a
harmonic-sum argument in which buckets near the threshold contribute little
while distant buckets contribute favorably (\Cref{lem:harmonic-drift}),
recombines these per-bucket pieces into the square-root rate for finitely
many thresholds (\Cref{thm:finite-v-family}); a finite-range assumption on
the comparator class then extends this exactly to every threshold at once
(\Cref{cor:v-family-full-range}), giving the square-root rate for the whole
bounded-proper-loss family (\Cref{thm:bounded-proper-rates}).  For convex
$1$-Lipschitz proper losses, a different positive decomposition into
clipped-ReLU losses (\Cref{prop:clipped-decomposition-front}) is rich enough
to represent the whole family while still admitting a finite cover with
useful metric entropy, and combining that cover with the offline conversion
recovers the sharper single-loss exponents for the entire family
(\Cref{thm:offline-convex-lipschitz}).  Canonical properization
(\Cref{lem:canonical-properization}) then carries every one of these
guarantees past proper losses, since best-response postprocessing of an
arbitrary bounded loss reduces exactly to swap-agnostic learning of one
proper loss with redundant messages.

It remains to produce the online guarantee itself and to carry it from the
online setting to the offline one.  The algorithm is a multiplicative-weights
implementation of Blackwell approachability \citep{blackwell1956analog}:
weights on tests summarize the direction in which the transcript so far most
needs correcting, and the next prediction is chosen to make progress against
that direction.  Two adjustments turn this into the second-order guarantee
we need.  Predictions must live on a finite grid, so the algorithm
randomizes over two adjacent grid points, which behaves like the desired
real-valued forecast and costs only the usual grid-rounding term.  And the
potential itself is adjusted with a Bernstein correction that tracks not
just cumulative residual error but also the squared size of the active
tests -- this is what turns an ordinary online guarantee into the offset
form that \eqref{eq:intro-identity} needs.  For the nonsmooth losses above,
we run this same argument once more over a geometric grid of learning
rates, so that whichever scale turns out to suit a given test's realized
mass, some rate on the grid is already close to it, without the algorithm
needing to know that scale in advance (\Cref{cor:sqrt-somc}).  Finally, the
offline learner simply runs this online algorithm on an i.i.d.\ sample,
records the sequence of predictors it produces, and returns one of them at
random; a Freedman--Bernstein concentration argument transfers the realized
online bias and mass bounds to their distributional counterparts
(\Cref{thm:online-to-batch}), and infinite test classes are handled along
the way by replacing them with a finite cover and paying the corresponding
approximation error (\Cref{lem:cover-transfer}).

\subsection{Relation to prior work}
\label{sec:intro-related-work}

Calibration of probabilistic forecasts dates back to \citet{dawid1982well},
with online calibration developed by \citet{foster1998asymptotic}.
Multicalibration asks for small residuals on every set in a rich class
\citep{hebertjohnson2018multicalibration}, multiaccuracy for the
corresponding unconditional guarantee \citep{kim2019multiaccuracy}, and
omniprediction uses such residual information to support many downstream
objectives with one predictor \citep{gopalan2022omnipredictors}.
\citet{collina2026samplecomplexity} pin down the minimax sample complexity
of batch multicalibration via an online-to-batch reduction, the same
strategy our offline conversion (\Cref{sec:offline}) uses to turn an online
second-order guarantee into a fast batch rate; \citet{gibbs2026sampleefficient}
show that omniprediction over a family of proper losses is strictly easier
than boosting-based multicalibration by exploiting proper-loss structure
directly, in the same spirit as our use of the drift term below.  Ordinary
multicalibration controls residual correlation by a common tolerance or by
bucket size, and summing such bounds over swap buckets is what makes swap
guarantees expensive; second-order multicalibration instead controls
residual correlation relative to a swap's own squared activity, the same
activity that produces the favorable drift we exploit throughout this
paper.

Three benchmarks are in play, ordered by how much the comparator may depend
on the prediction: agnostic learning fixes one hypothesis; our
swap-agnostic benchmark fixes the loss but lets the comparator hypothesis
vary with the prediction; and swap omniprediction lets both the loss and
the comparator vary.  \citet{gopalan2023swap} characterize swap
omniprediction through swap multicalibration, with algorithmic and
quantitative improvements in \citet{garg2024oracle,luo2025improved}.  Our
improvement comes from fixing the loss -- which supplies one coherent drift
term across all buckets -- rather than from improving rates for swap
omniprediction, so our results do not contradict lower bounds there.
\citet{hu2026nearoptimal} obtain near-optimal $\widetilde O(\sqrt T)$ swap
regret against an arbitrary Lipschitz convex loss revealed online, and
\citet{luo2025klcalibration} obtain a closely related $\widetilde
O(T^{1/3})$ rate for any twice-differentiable proper loss via a new
KL-calibration measure that generalizes a rate previously known only for
the squared loss \citep{fishelson2025full}.  Both measure regret against an
arbitrary postprocessing of the prediction rather than a hypothesis class,
so their bounds carry no $\log|\cH|$ term; our matching $T^{1/3}$ exponent
instead carries an explicit $(\log|\cH|)^{2/3}$ factor for the
hypothesis-valued comparator, obtained through second-order multicalibration
rather than calibration error.

Blackwell approachability steers vector-valued averages toward a convex
target set \citep{blackwell1956analog}, typically implemented with
multiplicative weights \citep{cesabianchi2006prediction}.
\citet{okoroafor2024recalibration} use approachability to recalibrate an
online predictor with an adjustable calibration-regret tradeoff, recent
omniprediction algorithms use a halfspace-oracle view of it to control many
residual correlations at once \citep{okoroafor2025nearoptimal}, and
\citet{hu2026simultaneous} develop a simultaneous approachability framework
for multiclass omniprediction.  We retain this architecture but adjust the
potential to yield a realized Bernstein-type inequality for every selector
test -- the source of the second-order term our offline reduction needs.
Proper losses are governed by the geometry of their Bayes risks
\citep{gneiting2007strictly,reid2010composite}, which we exploit through the
quadratic-growth inequality and through positive decompositions of broad
loss families.  U-calibration studies regret against an unknown downstream
agent using the same V-shaped basis for bounded proper losses
\citep{kleinberg2023ucalibration,luo2024optimal,li2022optimization}, and
\citet{hu2024payoffbounded} study predictions minimizing swap regret for all
payoff-bounded downstream tasks; our simultaneous bounded-proper-loss result
is a hypothesis-class analogue, replacing their arbitrary downstream action
with a hypothesis-valued rule chosen after the prediction, which their
techniques do not reach directly since the comparison must control the
whole hypothesis-swap selector rather than only the downstream task induced
by the scalar prediction.

\section{Preliminaries}

This section fixes notation for the offline and online swap-agnostic
benchmarks and records the one algebraic fact about proper losses that
drives every result later in the paper.

\subsection{Swap Agnostic Learning}

Fix an instance space $\mathcal X$ and the prediction grid
\[
  \Gamma_N=\left\{0,\frac1N,\ldots,1\right\}
  =\{\gam_0,\ldots,\gam_N\},\qquad \gam_i=\frac{i}{N}.
\]
An offline predictor is allowed to be randomized.  Formally, a predictor
$\hat p$ maps $x$ to a distribution over $\Gamma_N$; on a fresh example
$(X,Y)\sim\cD$, it draws $p\sim\hat p(X)$.  A hypothesis-swap rule is a map
\[
  \phi:\Gamma_N\to\cH,
\]
and, for a grid point $\gam_i$, write $h_i=\phi(\gam_i)$.  For a proper loss
$\ell$, define the offline swap-agnostic excess risk
\begin{equation}
  \SwapAgn_\cD(\ell,\cH)
  =\sup_{\phi:\Gamma_N\to\cH}
  \E_{(X,Y)\sim\cD,\,p\sim\hat p(X)}\left[
    \ell(p,Y)
    -\ell(\phi(p)(X),Y)
  \right].
  \label{eq:pop-swap}
\end{equation}
The predictor $\hat p$ is implicit in $\SwapAgn$, just as the realized
transcript is implicit in $\SwapReg$.

\subsection{Online Protocol and Swap Regret}

Fix a horizon $T$ and binary outcomes $y_t\in\{0,1\}$.
At round $t$:
\begin{enumerate}[label=(\roman*)]
  \item Nature reveals $x_t\in\mathcal X$.
  \item The learner announces a distribution $\psi_t$ over $\Gamma_N$.
  \item The learner privately draws $p_t\sim\psi_t$.
  \item Nature reveals $y_t\in\{0,1\}$.
\end{enumerate}
We allow a nonanticipating adaptive environment: conditional on the past and
on the announced distribution $\psi_t$, the outcome may be chosen
adversarially, but it cannot depend on the learner's fresh private draw
$p_t$.

For $i\in\{0,\ldots,N\}$, define the prediction bucket
\[
  I_i=\{t\in[T]:p_t=\gam_i\}.
\]
Let $\cH\subseteq\{h:\mathcal X\to[0,1]\}$ be a hypothesis class.  Unless
explicitly stated otherwise, online results assume that $\cH$ and all online
test classes are finite.  For a proper loss $\ell$, define
\begin{equation}
  \SwapReg_T(\ell,\phi)
  =\sum_{t=1}^T
  \left[
    \ell(p_t,y_t)-\ell(\phi(p_t)(x_t),y_t)
  \right],
  \qquad
  \SwapReg_T(\ell,\cH)
  =\sup_{\phi:\Gamma_N\to\cH}\SwapReg_T(\ell,\phi).
  \label{eq:swap-reg-def}
\end{equation}

\subsection{Conditional risks and proper losses}
\label{subsec:proper-losses}

For any binary loss $\ell:[0,1]\times\{0,1\}\to\R$, define the conditional
risk and outcome slope
\begin{equation}
  L_p(q)=(1-p)\ell(q,0)+p\ell(q,1),
  \qquad
  \Delta_\ell(q)=\ell(q,1)-\ell(q,0).
  \label{eq:risk-slope}
\end{equation}
The loss is \emph{proper} if $p\in\argmin_qL_p(q)$ for every $p$.  Its Bayes
risk is $L_p(p)$, and its negative Bayes risk is
\begin{equation}
  F_\ell(p)=-L_p(p).
  \label{eq:negative-bayes-risk}
\end{equation}
Properness implies that the Bayes risk $p\mapsto L_p(p)$ is concave and
$F_\ell$ is convex.

The next lemma is the single algebraic identity behind every reduction in
this paper: comparing a truthful prediction $p$ to any alternative
prediction $q$ splits the loss difference into a bias term, proportional to
the calibration residual $y-p$, and a nonnegative drift term that always
favors the truthful prediction.

\begin{lemma}[Proper-loss identity]
\label{lem:proper-identity}
For every proper loss $\ell$, every $p,q\in[0,1]$, and every
$y\in\{0,1\}$,
\begin{equation}
  \ell(p,y)-\ell(q,y)
  =(y-p)(\Delta_\ell(p)-\Delta_\ell(q))-[L_p(q)-L_p(p)],
  \label{eq:proper-identity}
\end{equation}
and $L_p(q)-L_p(p)\ge0$.  Moreover,
\[
  -\Delta_\ell(q)\in\partial F_\ell(q),
\]
so if $F_\ell$ is differentiable then $F_\ell'=-\Delta_\ell$ and
$L_p(q)-L_p(p)=D_{F_\ell}(p,q)$, the Bregman divergence of $F_\ell$.
\end{lemma}

\begin{proof}[Proof of \Cref{lem:proper-identity}]
Since $y\in\{0,1\}$,
\[
  \ell(q,y)=\ell(q,0)+y\Delta_\ell(q)
  =L_q(q)+(y-q)\Delta_\ell(q).
\]
Taking expectation under $Y\sim\Ber(p)$ gives
\begin{equation}
  L_p(q)=L_q(q)+(p-q)\Delta_\ell(q).
  \label{eq:proper-div-formula}
\end{equation}
Rearranging that formula and subtracting the canonical representations for
$p$ and $q$ proves \eqref{eq:proper-identity}.  Properness implies
\[
  L_p(p)
  \le L_q(q)+(p-q)\Delta_\ell(q),
\]
which is equivalent to $-\Delta_\ell(q)\in\partial F_\ell(q)$, and combined
with \eqref{eq:proper-div-formula} gives $L_p(q)-L_p(p)\ge0$ and the
Bregman-divergence identity when the subgradient is unique.
\end{proof}

\subsection{Finite covers}

For a class $\cG$ of bounded tests, a finite set $\widetilde\cG$ is a
uniform $\rho$-cover if every $g\in\cG$ has a
$\widetilde g\in\widetilde\cG$ satisfying
\[
  \sup_{x,p}|g(x,p)-\widetilde g(x,p)|\le\rho.
\]
We write $\mathcal N_\infty(\rho,\cG)$ for the smallest cover size.  Online
algorithms in this paper are only run on finite classes.  Infinite classes are
addressed offline by selecting a finite cover before training.

\section{Second-Order Multicalibration}
\label{sec:somc}

\Cref{lem:proper-identity} shows that swap-agnostic learning for a proper
loss reduces to controlling the residual correlation of tests induced by
swap selectors.  The right control is not a uniform correlation tolerance; it
compares the residual correlation of a test to its realized squared
activity, which is exactly what lets the drift term in
\eqref{eq:proper-identity} absorb the quadratic part of the error.  This
section develops that control on its own terms, entirely independently of
any particular loss: \Cref{def:somc} defines second-order multicalibration
online and distributionally, \Cref{sec:somc-algorithm} gives the
approachability algorithm and its high-probability online guarantee
\Cref{thm:online-somc}, \Cref{sec:somc-sqrt} records a square-root form
needed later for nonsmooth losses, and \Cref{sec:offline} converts the
online guarantee into the distributional guarantee
\Cref{thm:online-to-batch} used by the offline learners.  How a particular
loss family turns this generic control into a swap-regret bound is worked out
loss family by loss family starting in \Cref{sec:proper-families}.

For a transcript and a test $g:\mathcal X\times\Gamma_N\to[-1,1]$, define
\begin{equation}
  \Bias_T(g)=\sum_{t=1}^T(y_t-p_t)g(x_t,p_t),
  \qquad
  \Mass_T(g)=\sum_{t=1}^Tg(x_t,p_t)^2.
  \label{eq:online-bias-mass}
\end{equation}
For a randomized offline predictor $\hat p$, define the distributional
analogues
\begin{equation}
  \Bias_\cD(g;\hat p)=\E_{(X,Y)\sim\cD,\,p\sim\hat p(X)}[(Y-p)g(X,p)],
  \qquad
  \Mass_\cD(g;\hat p)=\E_{(X,Y)\sim\cD,\,p\sim\hat p(X)}[g(X,p)^2].
  \label{eq:pop-global-bias-mass}
\end{equation}

\begin{definition}[Second-order multicalibration]
\label{def:somc}
Let $\cG\subseteq\{g:\mathcal X\times\Gamma_N\to[-1,1]\}$.
A transcript satisfies $(c,a)$ second-order multicalibration with respect to
$\cG$ if
\begin{equation}
  \Bias_T(g)\le c\,\Mass_T(g)+a
  \label{eq:somc}
\end{equation}
for every $g\in\cG$.
A randomized online algorithm satisfies the guarantee with confidence
$1-\delta$ if the simultaneous event \eqref{eq:somc} has probability at least
$1-\delta$.
A randomized predictor $\hat p$ satisfies \emph{distributional} $(c,a)$
second-order multicalibration with respect to $\cG$ if, for every $g\in\cG$,
\begin{equation}
  \Bias_\cD(g;\hat p)\le c\,\Mass_\cD(g;\hat p)+a.
  \label{eq:pop-somc}
\end{equation}
In either case, the two-sided version is obtained by including both $g$ and
$-g$ in the test class.
\end{definition}

Let $\Phi_N=\cH^{N+1}$ denote the class of hypothesis-swap selectors.  A test
class induced by selectors always has the form
$\{g_\phi:\phi\in\Phi_N\}$ for some fixed
$g_\phi(x,\gam_i)=f(x,\gam_i,h_i)$ depending on $\phi$ only through
$h_i=\phi(\gam_i)$; every selector class used in this paper is of this form.
We record once, generically, the computational fact that lets the
approachability algorithm below run in time linear in $\cH$ rather than
exponential in the grid size.

\begin{proposition}[Product-class factorization]
\label{prop:factorization}
Let $g_\phi(x,\gam_i)=f(x,\gam_i,h_i)$ for $\phi\in\Phi_N$ as above.  Running
\Cref{alg:approachability} on $\{g_\phi:\phi\in\Phi_N\}$, the weights
factorize as
\[
  W_t(\phi)=\prod_{i=0}^NW_{t,i}(h_i).
\]
Consequently, the algorithm can be implemented with $(N+1)|\cH|$ scalar
weights rather than $|\cH|^{N+1}$ weights.
\end{proposition}

\begin{proof}[Proof of \Cref{prop:factorization}]
Initialize $W_{1,i}(h)=|\cH|^{-1}$ for every bucket $i$ and hypothesis $h$,
so
\[
  W_1(\phi)=|\cH|^{-(N+1)}
  =\prod_{i=0}^NW_{1,i}(h_i).
\]
Suppose $W_t(\phi)=\prod_iW_{t,i}(h_i)$.  If $p_t=\gam_j$, then
$g_\phi(x_t,p_t)=f(x_t,\gam_j,h_j)$ depends on $\phi$ only through its $j$th
coordinate.  Define
\[
  W_{t+1,j}(h)
  =W_{t,j}(h)
  \exp\!\left(
    \eta(y_t-\gam_j)f(x_t,\gam_j,h)
    -2\eta^2f(x_t,\gam_j,h)^2
  \right),
\]
and set $W_{t+1,i}=W_{t,i}$ for $i\ne j$.  Then
$W_{t+1}(\phi)=\prod_iW_{t+1,i}(h_i)$.

The normalizing constant also factorizes:
\[
  \sum_{\phi\in\Phi_N}W_t(\phi)=\prod_{i=0}^N\sum_{h\in\cH}W_{t,i}(h).
\]
Thus the normalized distribution over selectors has independent marginals
over buckets.  In particular, when evaluating a candidate prediction
$p=\gam_i$, the weighted averages $F_t(\gam_i)$ and $A_t(\gam_i)$ below only
require the marginal distribution on the $i$th hypothesis coordinate.
\end{proof}

\subsection{Algorithm}
\label{sec:somc-algorithm}

At a high level, the algorithm is a halfspace-oracle implementation of
Blackwell approachability, implemented with multiplicative weights over the
test class.

Fix a horizon $T$, grid size $N$, finite test class
$\cG\subseteq\{g:\mathcal X\times\Gamma_N\to[-1,1]\}$, and learning rate
$\eta\in(0,1/4]$.  The learner maintains one weight per test.  After
observing $x_t$, it normalizes the weights and computes, for each grid point,
\begin{align}
  F_t(\gam_i)
  &=\sum_{g\in\cG}w_t(g)g(x_t,\gam_i),
  \label{eq:F-def}\\
  A_t(\gam_i)
  &=\sum_{g\in\cG}w_t(g)g(x_t,\gam_i)^2.
  \label{eq:G-def}
\end{align}
The scalar function $F_t$ is the current halfspace direction used to choose
the next prediction.  The finite-grid halfspace oracle chooses a randomized
prediction $\psi_t$ as follows.  If
$F_t(0)\le0$, set $\psi_t=\delta_0$.  If $F_t(1)\ge0$, set
$\psi_t=\delta_1$.  Otherwise choose adjacent grid points
$a<b=a+1/N$ with $F_t(a)\ge0\ge F_t(b)$.  If one endpoint has value zero, put
all mass on that endpoint.  If both inequalities are strict, set
\[
  \lambda_t=\frac{F_t(a)}{F_t(a)-F_t(b)}
\]
and put mass $1-\lambda_t$ on $a$ and mass $\lambda_t$ on $b$.  Then
\[
  \E_{p\sim\psi_t}F_t(p)=0.
\]
This is the discrete analogue of choosing a point in $[0,1]$ where the
weighted residual direction crosses zero.

\begin{algorithm}[t]
\caption{Second-order multicalibration by approachability}
\label{alg:approachability}
\begin{algorithmic}[1]
\Require Horizon $T$, grid size $N$, finite test class $\cG$, rate
$\eta\in(0,1/4]$
\State Initialize $W_1(g)=1/|\cG|$ for all $g\in\cG$
\For{$t=1,\ldots,T$}
  \State Observe $x_t$ and normalize $w_t\propto W_t$
  \State Compute $F_t$ and $A_t$ from \eqref{eq:F-def}--\eqref{eq:G-def}
  \If{$F_t(0)\le0$}
    \State $\psi_t\gets\delta_0$
  \ElsIf{$F_t(1)\ge0$}
    \State $\psi_t\gets\delta_1$
  \Else
    \State Choose adjacent $a<b=a+1/N$ with $F_t(a)\ge0\ge F_t(b)$
    \If{$F_t(a)=0$ or $F_t(b)=0$}
      \State $\psi_t\gets$ point mass on a zero endpoint
    \Else
      \State $\lambda_t\gets F_t(a)/(F_t(a)-F_t(b))$
      \State $\psi_t\gets(1-\lambda_t)\delta_a+\lambda_t\delta_b$
    \EndIf
  \EndIf
  \State Draw $p_t\sim\psi_t$ and observe $y_t$
  \State Update
  \[
    W_{t+1}(g)
    \gets W_t(g)
    \exp\!\left(
      \eta(y_t-p_t)g(x_t,p_t)-2\eta^2g(x_t,p_t)^2
    \right)
  \]
\EndFor
\end{algorithmic}
\end{algorithm}

The following theorem is the main guarantee of this subsection.  Its proof
has two ingredients.  First, the randomized adjacent-grid step controls the
linear bias direction up to one global grid-rounding error.  Second, a
Bernstein-corrected exponential potential turns this one-step control into a
high-probability offset guarantee for every test.

\begin{theorem}[High-probability second-order multicalibration]
\label{thm:online-somc}
Run \Cref{alg:approachability} for $T$ rounds on $\cG$ with rate $\eta$.
Then, with probability at least $1-\delta$, simultaneously for every
$g\in\cG$,
\begin{equation}
  \boxed{
  \Bias_T(g)
  \le2\eta\Mass_T(g)
  +\frac{\log(|\cG|/\delta)+2T/N^2}{\eta}.}
  \label{eq:online-somc-main}
\end{equation}
\end{theorem}

\paragraph{The grid halfspace oracle.}

Let $\rho=1/N$.  Given a distribution $w$ over $\cG$ and a context $x$, define
\[
  F(\gam_i)=\sum_{g\in\cG}w(g)g(x,\gam_i),
  \qquad
  A(\gam_i)=\sum_{g\in\cG}w(g)g(x,\gam_i)^2.
\]

\begin{lemma}[Algorithmic grid oracle with quadratic slack]
\label{lem:grid-root}
Let $\psi$ be the distribution constructed by the grid-oracle step of
\Cref{alg:approachability} from the values of $F$ on $\Gamma_N$.  Then, for
every $y\in[0,1]$,
\begin{equation}
  \E_{p\sim\psi}[(y-p)F(p)]
  \le\rho\sqrt{\E_{p\sim\psi}A(p)}.
  \label{eq:grid-root}
\end{equation}
\end{lemma}

\begin{proof}[Proof of \Cref{lem:grid-root}]
If $F(0)\le0$, the algorithm sets $\psi=\delta_0$.  Then
$\E[(y-p)F(p)]=yF(0)\le0$, while the right-hand side of
\eqref{eq:grid-root} is nonnegative.  If $F(1)\ge0$, the algorithm sets
$\psi=\delta_1$, and $(y-1)F(1)\le0$.

It remains to handle the crossing case.  The algorithm chooses adjacent grid
points $a<b=a+\rho$ with $F(a)\ge0\ge F(b)$.  If $F(a)=0$ or $F(b)=0$, it
puts all mass on a zero endpoint.  Otherwise it sets
\[
  \lambda=\frac{F(a)}{F(a)-F(b)}\in(0,1),
\]
so that
\[
  (1-\lambda)F(a)+\lambda F(b)=0,
\]
and let $\psi=(1-\lambda)\delta_a+\lambda\delta_b$.  Writing
$(1-\lambda)F(a)=-\lambda F(b)$, we get
\[
\begin{aligned}
  \E[(y-p)F(p)]
  &=(1-\lambda)F(a)(y-a)+\lambda F(b)(y-b)\\
  &=(b-a)(1-\lambda)F(a)
    =\rho(1-\lambda)F(a)\\
  &\le\rho\sqrt{(1-\lambda)A(a)}
    \le\rho\sqrt{\E A(p)}.
\end{aligned}
\]
\end{proof}

\paragraph{The exponential potential.}

\begin{lemma}[Offset exponential inequality]
\label{lem:exp-ineq}
For $\eta\in(0,1/4]$, $g\in[-1,1]$, and $z\in[-|g|,|g|]$,
\begin{equation}
  \exp(\eta z-2\eta^2g^2)
  \le1+\eta z-\frac{\eta^2}{4}g^2.
  \label{eq:exp-ineq}
\end{equation}
\end{lemma}

\begin{proof}[Proof of \Cref{lem:exp-ineq}]
Set $u=\eta z$ and $v=\eta^2g^2$.  Then $|u|\le1/4$, $u^2\le v$, and
$v\le1/16$.  The standard Bernstein exponential inequality
$e^{u-u^2}\le1+u$ for $u\ge-1/2$ holds, for example, as in
\citet[Chapter~2]{cesabianchi2006prediction}.  Since
$r=2v-u^2$ lies in $[0,1/8]$, we also have $e^{-r}\le1-r/2$.  Therefore
\[
  e^{u-2v}
  =e^{u-u^2}e^{-(2v-u^2)}
  \le(1+u)\left(1-\frac{2v-u^2}{2}\right).
\]
Because $u^2\le v$, the last display is at most
\[
  (1+u)\left(1-\frac v2\right)
  =1+u-\frac v2-\frac{uv}{2}
  \le1+u-\frac v4,
\]
where the final inequality uses $u\ge-1/2$.  Substituting back
$u=\eta z$ and $v=\eta^2g^2$ proves \eqref{eq:exp-ineq}.
\end{proof}

\paragraph{Completing the proof.}

\begin{proof}[Proof of \Cref{thm:online-somc}]
Let $Z_t=\sum_{g\in\cG}W_t(g)$.  Conditional on the past, $x_t$, the
announced distribution, and the nonanticipating value of $y_t$, but not the
fresh draw $p_t$, \Cref{lem:exp-ineq} gives
\[
  \E_t\frac{Z_{t+1}}{Z_t}
  \le
  1+\eta\E_t[(y_t-p_t)F_t(p_t)]
  -\frac{\eta^2}{4}\E_t A_t(p_t).
\]
By \Cref{lem:grid-root},
\[
  \eta\E_t[(y_t-p_t)F_t(p_t)]
  \le
  \eta\rho\sqrt{\E_tA_t(p_t)}
  \le
  \frac{\eta^2}{8}\E_tA_t(p_t)+2\rho^2.
\]
Thus $\E_t[Z_{t+1}/Z_t]\le e^{2\rho^2}$, and
$e^{-2\rho^2(t-1)}Z_t$ is a nonnegative supermartingale.  Markov's
inequality gives $Z_{T+1}\le\delta^{-1}e^{2T/N^2}$ with probability at
least $1-\delta$.  For each $g$,
\[
  \frac1{|\cG|}
  \exp\!\left(\eta\Bias_T(g)-2\eta^2\Mass_T(g)\right)
  \le Z_{T+1}.
\]
Taking logarithms and dividing by $\eta$ proves
\eqref{eq:online-somc-main}.
\end{proof}

\subsection{Square-root form by scaling the test class}
\label{sec:somc-sqrt}

Nonsmooth losses will need a version of \Cref{thm:online-somc} with no linear
term in $\Mass_T(g)$ at all, only a square root of it.  No separate
algorithm is needed for this: it is enough to run
\Cref{alg:approachability} at one fixed rate on a test class that already
contains every rescaling of $\cG$, and then optimize the scale after the
fact.  Let
\[
  \cA_T=\{2^{-j}:j=0,1,\ldots,\lceil\log_2T\rceil\}
\]
be a dyadic grid of scales between $1$ and about $1/T$, and let
\[
  \cG^{\mathrm{sc}}
  =\{s\alpha g:s\in\{-1,1\},\ g\in\cG,\ \alpha\in\cA_T\}
  \subseteq\{g:\mathcal X\times\Gamma_N\to[-1,1]\}.
\]
Since $\Bias_T(\alpha g)=\alpha\Bias_T(g)$ and
$\Mass_T(\alpha g)=\alpha^2\Mass_T(g)$, running
\Cref{alg:approachability} on $\cG^{\mathrm{sc}}$ with rate $\eta=1/4$ and
applying \Cref{thm:online-somc} gives, with probability at least $1-\delta$,
simultaneously for every $s,g,\alpha$,
\begin{equation}
  \alpha s\Bias_T(g)-\frac{\alpha^2}2\Mass_T(g)
  \le4\Lambda_T,
  \qquad
  \Lambda_T
  =1+\log\!\left(\frac{2|\cG||\cA_T|}{\delta}\right)+\frac{2T}{N^2}.
  \label{eq:sqrt-somc-scaled}
\end{equation}
Choosing $s$ to match the sign of $\Bias_T(g)$ turns this into, for every
$\alpha\in\cA_T$,
\begin{equation}
  |\Bias_T(g)|
  \le\frac\alpha2\Mass_T(g)+\frac{4\Lambda_T}\alpha.
  \label{eq:sqrt-somc-per-alpha}
\end{equation}

\begin{corollary}[Square-root second-order control]
\label{cor:sqrt-somc}
There is a universal constant $C$ such that, with probability at least
$1-\delta$, simultaneously for every $g\in\cG$,
\begin{equation}
  |\Bias_T(g)|
  \le C\sqrt{\Lambda_T\Mass_T(g)}+C\Lambda_T.
  \label{eq:sqrt-somc}
\end{equation}
\end{corollary}

\begin{proof}[Proof of \Cref{cor:sqrt-somc}]
Work on the event where \eqref{eq:sqrt-somc-per-alpha} holds for every
$\alpha\in\cA_T$, which has probability at least $1-\delta$.  Fix $g$ and
write $M=\Mass_T(g)\le T$.  If $M\le8\Lambda_T$, take $\alpha=1\in\cA_T$ in
\eqref{eq:sqrt-somc-per-alpha} to get $|\Bias_T(g)|\le4\Lambda_T+4\Lambda_T
=8\Lambda_T$.  Otherwise, the continuous minimizer of
$(\alpha/2)M+4\Lambda_T/\alpha$ over $\alpha>0$ is
$\alpha_*=\sqrt{8\Lambda_T/M}<1$, with minimal value
$2\sqrt{2\Lambda_TM}$.  Since $M\le T$ and $\Lambda_T\ge1$,
\[
  \alpha_*^2=\frac{8\Lambda_T}M\ge\frac{8\Lambda_T}T\ge\frac8T\ge\frac1{T^2},
\]
so $\alpha_*\ge1/T$, which is at least the smallest scale
$2^{-\lceil\log_2T\rceil}$ in $\cA_T$.  Hence some $\alpha\in\cA_T$ lies
within a factor of $2$ of $\alpha_*$; substituting such an $\alpha$ into
\eqref{eq:sqrt-somc-per-alpha} changes each term by at most a factor of $2$,
giving $|\Bias_T(g)|\le4\sqrt{2\Lambda_TM}$.  Both cases are covered by
\eqref{eq:sqrt-somc} for a suitable universal constant $C$.
\end{proof}

\subsection{Offline conversion}
\label{sec:offline}

We now convert online second-order multicalibration into an offline
guarantee.  The conversion returns a randomized predictor, which is standard
for online-to-batch reductions.  Unlike a naive holdout argument, it preserves
the sequential prediction rules whose realized residuals were controlled
online.

\begin{theorem}[High-probability offline second-order multicalibration]
\label{thm:online-to-batch}
Let $\cG\subseteq\{g:\mathcal X\times\Gamma_N\to[-1,1]\}$ be finite, let
$Z_1,\ldots,Z_m$ be i.i.d. examples from $\cD$, and run
\Cref{alg:approachability} sequentially on this sample with rate
$\eta\in(0,1/4]$.  Let $\widehat p$ be the randomized historical predictor
constructed below.  Then, with probability at least $1-\delta$ over the
training sample and all training randomization, $\widehat p$ satisfies
\begin{equation}
  \boxed{
  \Bias_\cD(g;\widehat p)
  \le5\eta\Mass_\cD(g;\widehat p)
  +\frac{
    2\log(12|\cG|/\delta)+2m/N^2
  }{\eta m}}
  \label{eq:offline-ca}
\end{equation}
simultaneously for every $g\in\cG$.
\end{theorem}

\paragraph{The randomized historical predictor.}

Let $Z_1,\ldots,Z_m$, with $Z_t=(X_t,Y_t)$, be i.i.d. from a distribution
$\cD$.  Run \Cref{alg:approachability} sequentially on this sample.  Before
round $t$, the history determines a predictor
\[
  \pi_t:\mathcal X\to\Delta(\Gamma_N),
\]
namely the distribution that the online algorithm would produce after seeing a
new context $x$ and the first $t-1$ examples.  Store these predictors.

The offline output $\widehat p$ is defined by:
\begin{enumerate}[label=(\roman*)]
  \item draw $J$ uniformly from $\{1,\ldots,m\}$;
  \item on input $x$, draw $p\sim\pi_J(x)$.
\end{enumerate}
The index $J$ and prediction draw are fresh randomization at test time.

\paragraph{A martingale Bernstein lemma.}

We use the following standard fixed-time form of Freedman's inequality.

\begin{lemma}[Freedman inequality]
\label{lem:freedman}
Let $(X_t,\mathcal F_t)_{t=1}^m$ be martingale differences satisfying
$|X_t|\le1$, and let
\[
  V_m=\sum_{t=1}^m\E[X_t^2\mid\mathcal F_{t-1}].
\]
For every $\lambda>0$, with probability at least $1-2e^{-\lambda}$,
\begin{equation}
  \abs{\sum_{t=1}^mX_t}
  \le\sqrt{2V_m\lambda}+\frac{2\lambda}{3}.
  \label{eq:freedman}
\end{equation}
\end{lemma}

\begin{proof}[Proof of \Cref{lem:freedman}]
This is the two-sided fixed-time Freedman--Bernstein inequality; see, for
example, \citet[Chapter~2]{cesabianchi2006prediction}.
\end{proof}

\paragraph{Distributional transfer.}

For a fixed test $g$, define on training round $t$
\[
  Z_t^g=(Y_t-p_t)g(X_t,p_t),
  \qquad
  R_t^g=g(X_t,p_t)^2.
\]
Let
\[
  \mu_t^g=\E[Z_t^g\mid\mathcal F_{t-1}],
  \qquad
  \nu_t^g=\E[R_t^g\mid\mathcal F_{t-1}],
\]
where the conditional expectation integrates over the fresh i.i.d. example
and the fresh prediction draw.  By construction of the historical predictor
mixture,
\begin{equation}
  \Bias_\cD(g;\widehat p)=\frac1m\sum_{t=1}^m\mu_t^g,
  \qquad
  \Mass_\cD(g;\widehat p)=\frac1m\sum_{t=1}^m\nu_t^g.
  \label{eq:historical-predictor-identities}
\end{equation}

\begin{proof}[Proof of \Cref{thm:online-to-batch}]
Let
\[
  \lambda=\log(12|\cG|/\delta),
  \qquad
  \Lambda_m=2\lambda+\frac{2m}{N^2},
\]
and fix $g\in\cG$.  Abbreviate
$Z_t=Z_t^g$, $R_t=R_t^g$, $\mu_t=\mu_t^g$, $\nu_t=\nu_t^g$, and
$V=\sum_t\nu_t$.  Since $R_t,\nu_t\in[0,1]$, $|R_t-\nu_t|\le1$, and
\Cref{lem:freedman} applies to $R_t-\nu_t$ directly.  Since $Z_t,\mu_t$ only
lie in $[-1,1]$, their difference satisfies $|Z_t-\mu_t|\le2$ rather than
the unit bound \Cref{lem:freedman} assumes, so we instead apply the lemma to
the rescaled difference $(Z_t-\mu_t)/2$ -- using
$\E[(Z_t-\mu_t)^2\mid\mathcal F_{t-1}]\le\E[Z_t^2\mid\mathcal F_{t-1}]
\le\nu_t$, since $Z_t^2\le R_t$ pointwise -- and multiply the resulting
inequality by $2$.  Union bound over $g$.  With probability at least
$1-\delta/3$, simultaneously for all $g$,
\[
  \sum_tR_t\le2V+2\lambda,
  \qquad
  \sum_t\mu_t
  \le\sum_tZ_t+\sqrt{2V\lambda}+\frac{4\lambda}{3}.
\]
By a union bound, \Cref{thm:online-somc} also holds on the training
transcript with failure probability $\delta/3$.  On the intersection of
these events,
\begin{align*}
  \sum_t\mu_t
  &\le2\eta\sum_tR_t
      +\frac{\log(3|\cG|/\delta)+2m/N^2}{\eta}
      +\sqrt{2V\lambda}+\frac{4\lambda}{3}\\
  &\le4\eta V+4\eta\lambda
      +\frac{\log(3|\cG|/\delta)+2m/N^2}{\eta}
      +\sqrt{2V\lambda}+\frac{4\lambda}{3}\\
  &\le5\eta V+\frac{\Lambda_m}{\eta}.
\end{align*}
The last line uses
$\sqrt{2V\lambda}\le\eta V+\lambda/(2\eta)$ and $\eta\le1/4$, which give
$4\eta\lambda+\lambda/(2\eta)+4\lambda/3\le2\lambda/\eta$ throughout this
range with room to spare.  Divide by
$m$ and use \eqref{eq:historical-predictor-identities}.
\end{proof}

Exactly as in \Cref{sec:somc-sqrt}, running the historical-predictor
construction on the signed, geometrically scaled test class
$\cG^{\mathrm{sc}}$ and optimizing the scale after the fact removes the
linear $\Mass_\cD$ term at the cost of a square root.

\begin{corollary}[Offline square-root second-order control]
\label{cor:offline-sqrt-somc}
There is a universal constant $C$ such that, with probability at least
$1-\delta$, simultaneously for every $g\in\cG$,
\begin{equation}
  |\Bias_\cD(g;\widehat p)|
  \le
  C\sqrt{\frac{\Lambda_m}m\Mass_\cD(g;\widehat p)}
  +C\frac{\Lambda_m}m,
  \label{eq:offline-sqrt-somc}
\end{equation}
where $\Lambda_m=1+2\log(24|\cG||\cA_m|/\delta)+2m/N^2$ and
$\cA_m=\{2^{-j}:0\le j\le\lceil\log_2m\rceil\}$.
\end{corollary}

\begin{proof}[Proof of \Cref{cor:offline-sqrt-somc}]
Apply \Cref{thm:online-to-batch} with rate $\eta=1/4$ to
$\cG^{\mathrm{sc}}=\{s\alpha g:s\in\{-1,1\},g\in\cG,\alpha\in\cA_m\}$, whose
cardinality is $2|\cG||\cA_m|$, using
$\Bias_\cD(\alpha g;\widehat p)=\alpha\Bias_\cD(g;\widehat p)$ and
$\Mass_\cD(\alpha g;\widehat p)=\alpha^2\Mass_\cD(g;\widehat p)$.  This gives,
with probability at least $1-\delta$, simultaneously for every $s,g,\alpha$,
\[
  \alpha s\Bias_\cD(g;\widehat p)
  \le\frac54\alpha^2\Mass_\cD(g;\widehat p)+\frac{4\Lambda_m}m,
\]
using $2\log(12|\cG^{\mathrm{sc}}|/\delta)\le\Lambda_m$.  Choosing $s$ to
match the sign of $\Bias_\cD(g;\widehat p)$ and writing
$M=\Mass_\cD(g;\widehat p)\le1$ gives, for every $\alpha\in\cA_m$,
\[
  |\Bias_\cD(g;\widehat p)|
  \le\frac{5\alpha}4M+\frac{4\Lambda_m}{\alpha m}
  =\frac1m\left(\frac{5\alpha}4M'+\frac{4\Lambda_m}\alpha\right),
  \qquad M'=mM\le m.
\]
The bracketed expression is exactly of the form optimized in the proof of
\Cref{cor:sqrt-somc}, with $m$ in place of $T$, $M'$ in place of
$\Mass_T(g)$, and $\Lambda_m$ in place of $\Lambda_T$ (the constant $5/4$ in
place of $1/2$ only changes the case-split threshold and the universal
constant, not the argument).  That optimization over $\alpha\in\cA_m$ bounds
the bracket by $C\sqrt{\Lambda_mM'}=C\sqrt{\Lambda_mmM}$ for a suitable
universal $C$, and dividing by $m$ gives
\[
  |\Bias_\cD(g;\widehat p)|
  \le\frac{C\sqrt{\Lambda_mmM}}m
  =C\sqrt{\frac{\Lambda_m}mM},
\]
which is \eqref{eq:offline-sqrt-somc} (the additive term $C\Lambda_m/m$
appears in the same way in the small-$M'$ case of that optimization).
\end{proof}

The offline theorem is stated for a finite test class, but the guarantee
extends from a finite uniform cover to the full class with only the following
approximation cost.

\begin{lemma}[Uniform-cover transfer]
\label{lem:cover-transfer}
Let $\widetilde\cG$ be a uniform $\rho$-cover of a class $\cG$.  If a
randomized predictor $\hat p$ satisfies distributional $(c,a)$ second-order
multicalibration for $\widetilde\cG$, then it satisfies distributional
$(2c,\,a+(2c+1)\rho)$ second-order multicalibration for $\cG$.
\end{lemma}

\begin{proof}[Proof of \Cref{lem:cover-transfer}]
Fix $g\in\cG$ and choose $\widetilde g$ with
$\|g-\widetilde g\|_\infty\le\rho$.  Since $|Y-p|\le1$,
\[
  \Bias_\cD(g)\le\Bias_\cD(\widetilde g)+\rho.
\]
Also $\widetilde g^2\le2g^2+2\rho^2\le2g^2+2\rho$, and hence
\[
  \Bias_\cD(g)\le2c\Mass_\cD(g)+a+(2c+1)\rho.
\]
\end{proof}

\section{Lipschitz Proper Losses}
\label{sec:proper-families}

This section instantiates the generic second-order control of
\Cref{sec:somc} for a first, concrete family of proper losses: those whose
partial losses are Lipschitz.  We start with a single fixed Lipschitz proper
loss, pass to finite families of such losses, and then, in the next
subsection, use monotone transforms to reach every Lipschitz proper loss at
once.  Two further families -- bounded proper losses and convex Lipschitz
proper losses -- are treated in the sections that follow this one.

The section's key inequality -- and, later, the key inequality of the convex
Lipschitz family in \Cref{sec:convex-lipschitz} -- both come from one fact
about convex functions in general: the curvature of a convex function
controls how much its Bregman divergence must grow away from a point.  We
record this fact once, as a lemma about a general convex $F$, and specialize
it below to the negative Bayes risk $F_\ell$ to obtain
\Cref{lem:quadratic-growth}, the inequality this section is built around.

\begin{lemma}[Smooth Bregman lower bound]
\label{lem:smooth-bregman}
Let $F:[0,1]\to\R$ be convex and differentiable, and suppose $F'$ is
$W$-Lipschitz.  Then
\begin{equation}
  D_F(p,q)
  \ge\frac{(F'(p)-F'(q))^2}{2W}
  \label{eq:smooth-bregman}
\end{equation}
for every $p,q\in[0,1]$.
\end{lemma}

\begin{proof}[Proof of \Cref{lem:smooth-bregman}]
Assume $p>q$ and let $a=F'(p)-F'(q)\ge0$.  For $u\in[q,p]$,
\[
  F'(u)-F'(q)
  \ge\max\{0,a-W(p-u)\}.
\]
Since $a\le W(p-q)$,
\begin{align*}
  D_F(p,q)
  &=\int_q^p(F'(u)-F'(q))\,du\\
  &\ge\int_{p-a/W}^p(a-W(p-u))\,du
  =\frac{a^2}{2W}.
\end{align*}
The case $p<q$ is symmetric.
\end{proof}

We call a proper loss $\ell$ \emph{$L$-Lipschitz} if both partial losses
satisfy $|\ell(p,y)-\ell(q,y)|\le L|p-q|$ for $y\in\{0,1\}$.  The next lemma
is the key loss-specific inequality used throughout this section: it upgrades
the identity of \Cref{lem:proper-identity} to a quantitative statement in
which the drift term controls the square of the bias coefficient, with an
explicit constant governed by $L$.

\begin{lemma}[Quadratic-growth inequality]
\label{lem:quadratic-growth}
If $\ell$ is an $L$-Lipschitz proper loss, then $F_\ell$ is differentiable,
$F_\ell'=-\Delta_\ell$ is $2L$-Lipschitz, and for every $p,q\in[0,1]$,
\begin{equation}
  L_p(q)-L_p(p)
  \ge\frac1{4L}(\Delta_\ell(p)-\Delta_\ell(q))^2.
  \label{eq:quadratic-growth}
\end{equation}
Consequently,
\begin{equation}
  \ell(p,y)-\ell(q,y)
  \le(y-p)(\Delta_\ell(p)-\Delta_\ell(q))
  -\frac1{4L}(\Delta_\ell(p)-\Delta_\ell(q))^2.
  \label{eq:lipschitz-key-inequality}
\end{equation}
Conversely, if \eqref{eq:quadratic-growth} holds with $4L$ replaced by a
constant $K$, then both partial losses are $K/2$-Lipschitz.
\end{lemma}

\begin{proof}[Proof of \Cref{lem:quadratic-growth}]
The triangle inequality gives
$|\Delta_\ell(p)-\Delta_\ell(q)|\le2L|p-q|$.  Thus, by
\Cref{lem:proper-identity}, $g(q)=-\Delta_\ell(q)$ is a continuous
$2L$-Lipschitz selection from $\partial F_\ell(q)$.  For $x<z$, the
subgradient inequalities give
$g(x)\le(F_\ell(z)-F_\ell(x))/(z-x)\le g(z)$; letting $z\downarrow x$ and
$z\uparrow x$ shows that $F_\ell$ is differentiable with $F_\ell'=g$, hence
$F_\ell'$ is $2L$-Lipschitz.  \Cref{lem:smooth-bregman} with $W=2L$, together
with the Bregman-divergence identity of \Cref{lem:proper-identity}, gives
\eqref{eq:quadratic-growth}; substituting into \Cref{lem:proper-identity}
proves \eqref{eq:lipschitz-key-inequality}.

For the converse, apply the assumed inequality to $(p,q)$ and to $(q,p)$ and
add.  The identity
\[
  [L_p(q)-L_p(p)]+[L_q(p)-L_q(q)]
  =-(p-q)(\Delta_\ell(p)-\Delta_\ell(q))
\]
and monotonicity of $\Delta_\ell$ (which is nonincreasing, since
$-\Delta_\ell\in\partial F_\ell$ and $F_\ell$ is convex) give
$|\Delta_\ell(p)-\Delta_\ell(q)|\le(K/2)|p-q|$.  The canonical formulas
$\ell(p,0)=L_p(p)-p\Delta_\ell(p)$,
$\ell(p,1)=L_p(p)+(1-p)\Delta_\ell(p)$ then imply that both
partial losses are $K/2$-Lipschitz.
\end{proof}

Lipschitzness of the partial losses is often introduced as a separate
condition from smoothness of the negative Bayes risk, but for proper losses
the two coincide up to a factor of $2$, so there is no need to treat them as
two different hypotheses.

\begin{corollary}[Lipschitzness and smoothness coincide]
\label{cor:lipschitz-smooth-equivalence}
If both partial losses of $\ell$ are $L$-Lipschitz, then $F_\ell'$ is
$2L$-Lipschitz (\Cref{lem:quadratic-growth}).  Conversely, if $F_\ell'$ is
$W$-Lipschitz, then both partial losses of $\ell$ are $W$-Lipschitz.
\end{corollary}

\begin{proof}[Proof of \Cref{cor:lipschitz-smooth-equivalence}]
Only the converse needs proof.  If $F_\ell'$ is $W$-Lipschitz, then
$F_\ell''$ exists almost everywhere with $|F_\ell''|\le W$.  Differentiating
the canonical formulas
$\ell(p,0)=-F_\ell(p)-pF_\ell'(p)$ and
$\ell(p,1)=-F_\ell(p)-(1-p)F_\ell'(p)$
gives $\ell(\cdot,0)'(p)=pF_\ell''(p)$ and
$\ell(\cdot,1)'(p)=-(1-p)F_\ell''(p)$, both bounded in absolute value by $W$
since $p,1-p\le1$.
\end{proof}

Convexity of the partial losses is not enough by itself to give any
quadratic-growth constant, however large -- Lipschitzness (equivalently,
smoothness) is doing real work above.

\begin{example}[Convex bounded proper loss with no quadratic-growth constant]
\label{ex:convex-infinite-kappa}
Let
\[
  w(p)=\frac1{\sqrt{p(1-p)}},
\]
and define
\[
  \ell(p,0)=\int_0^p u w(u)\,du,
  \qquad
  \ell(p,1)=\int_p^1(1-u)w(u)\,du.
\]
The loss is bounded, strictly proper, and both partial losses are convex.
Its negative Bayes risk has curvature $F''(p)=w(p)$, which diverges as
$p\to0$ or $p\to1$.  Consequently $F_\ell'=-\Delta_\ell$ is not
$W$-Lipschitz for any finite $W$, so no quadratic-growth inequality
$L_p(q)-L_p(p)\ge(1/K)(\Delta_\ell(p)-\Delta_\ell(q))^2$ holds for any finite
$K$: locally,
\[
  \frac{(F'(p)-F'(q))^2}{D_F(p,q)}\to2F''(q)\to\infty
  \qquad\text{as }q\to0\text{ or }q\to1.
\]
Thus convexity alone does not control the quadratic-growth constant, and
\Cref{lem:quadratic-growth} applies to this loss for no finite Lipschitz (or,
equivalently, smoothness) constant.
\end{example}

Consequently, every result in this section stated for an $L$-Lipschitz
proper loss applies verbatim to a proper loss whose negative Bayes risk is
$W$-smooth, simply by taking $L=W$; we do not restate the theorems
separately for smoothness.

\subsection{A single Lipschitz loss}

For a swap rule $\phi$, define
\[
  g_{\ell,\phi}(x,p)
  =\frac{\Delta_\ell(p)-\Delta_\ell(\phi(p)(x))}{2L}.
\]
Since $\ell$ is $L$-Lipschitz, these tests take values in $[-1,1]$.

\begin{theorem}[Single Lipschitz proper loss]
\label{thm:single-lipschitz}
Let $\ell$ be an $L$-Lipschitz proper loss and let $\cH$ be finite.  There is
a choice of the prediction grid such that, with probability at least
$1-\delta$,
\[
  \SwapReg_T(\ell,\cH)
  =\widetilde O\!\left(LT^{1/3}(\log|\cH|)^{2/3}\right).
\]
The historical predictor satisfies
\[
  \SwapAgn_\cD(\ell,\cH)
  =\widetilde O\!\left(L\left(\frac{\log|\cH|}{m}\right)^{2/3}\right).
\]
\end{theorem}

\begin{proof}[Proof of \Cref{thm:single-lipschitz}]
Fix a swap rule $\phi$.  By \Cref{lem:quadratic-growth},
\[
  \SwapReg_T(\ell,\phi)
  \le2L\Bias_T(g_{\ell,\phi})-L\Mass_T(g_{\ell,\phi}).
\]
Run \Cref{alg:approachability} on the selector class
$\{g_{\ell,\phi}:\phi\in\Phi_N\}$, of size at most $|\cH|^{N+1}$, with
$\eta=1/4$.  By \Cref{thm:online-somc}, the $\Mass_T$ term cancels exactly
and
\[
  \SwapReg_T(\ell,\phi)
  \le8L\left[(N+1)\log|\cH|+\log\frac1\delta+\frac{2T}{N^2}\right].
\]
This event is simultaneous over all swap rules, so it bounds
$\SwapReg_T(\ell,\cH)$.  Optimizing $N\asymp(T/\log|\cH|)^{1/3}$ proves the
online statement.

For the offline statement, apply \Cref{thm:online-to-batch} to the same
selector class with $\eta=1/10$.  Again the population mass term cancels,
leaving
\[
  \SwapAgn_\cD(\ell,\cH)
  =O\!\left(L\,\frac{(N+1)\log|\cH|+\log(1/\delta)+m/N^2}{m}\right).
\]
Optimize $N\asymp(m/\log|\cH|)^{1/3}$.
\end{proof}

By \Cref{prop:factorization}, the exponential weights on the selector class
factorize over prediction buckets, so the algorithm above uses
$(N+1)|\cH|$ scalar weights rather than $|\cH|^{N+1}$.

The online exponent in \Cref{thm:single-lipschitz} cannot be improved in
general.  The next result is stated for the half-Brier loss, and uses the
following sharp online recalibration lower bound as a black box.

\begin{theorem}[Recalibration lower bound, black-box form, \citep{hu2026recalibration}]
\label{thm:recalibration-lb}
There are constants $c_0,\eps_0>0$ such that for every
$\eps\le\eps_0$ and every $T\le c_0\eps^{-3}$, there is a nonanticipating
distribution over hints $q_t\in[1/4,3/4]$ and labels
$Y_t\sim\Ber(q_t)$ for which no forecaster simultaneously has
\[
  \E\mathcal K_1(p_{1:T},Y_{1:T})\le\eps
\]
and
\[
  \E\frac1T\sum_{t=1}^T
  \left[(p_t-Y_t)^2-(q_t-Y_t)^2\right]
  \le\eps^2.
\]
Here
\[
  \mathcal K_1(p_{1:T},y_{1:T})
  =\frac1T\sum_v\left|\sum_{t:p_t=v}(p_t-y_t)\right|.
\]
\end{theorem}

\begin{theorem}[Fixed three-hypothesis lower bound]
\label{thm:brier-lb}
Let
\[
  \ell_{\mathrm B}(p,y)=\frac12(p-y)^2.
\]
There is a fixed domain $\cX$ and a fixed class $\cH_\star$ of three
hypotheses such that every randomized learner satisfies
\[
  \sup_{(x_t,y_t)}
  \E\SwapReg_T(\ell_{\mathrm B},\cH_\star)
  \ge cT^{1/3}
\]
for all sufficiently large $T$.
\end{theorem}

\begin{proof}[Proof of \Cref{thm:brier-lb}]
Let
\[
  \cX=[1/4,3/4]\times[0,1/8]
\]
and
\[
  h_0(q,s)=q,
  \qquad
  h_+(q,s)=q+s,
  \qquad
  h_-(q,s)=q-s.
\]
Fix $\eps$ and put $s=\eps/2$.  Feed the learner contexts
$x_t=(q_t,s)$ from the hard recalibration instance.  Work first with the
unscaled square loss $\bar\ell(p,y)=(p-y)^2$, and let
$\overline{\SwapReg}_T$ denote the corresponding swap regret.  Then
$\overline{\SwapReg}_T=2\SwapReg_T(\ell_{\mathrm B},\cH_\star)$.

Let
\[
  R_0=\sum_t[(p_t-Y_t)^2-(q_t-Y_t)^2].
\]
The constant selector $h_0$ gives $\overline{\SwapReg}_T\ge R_0$, and
conditional unbiasedness gives
\[
  \E R_0=\E\sum_t(p_t-q_t)^2.
\]
For every realized prediction value $v$, let
$I_v=\{t:p_t=v\}$ and choose $h_+$ or $h_-$ according to the sign of
$\sum_{t\in I_v}(Y_t-q_t)$.  Since
\[
  (q-y)^2-(q+\tau s-y)^2
  =2\tau s(y-q)-s^2,
\]
this selector gives
\[
  \overline{\SwapReg}_T
  \ge R_0+2sM_T-Ts^2,
  \qquad
  M_T=\sum_v\left|\sum_{t\in I_v}(Y_t-q_t)\right|.
\]
If $\E\overline{\SwapReg}_T\le rT$, then
\[
  \E\sum_t(p_t-q_t)^2\le rT,
  \qquad
  \frac{\E M_T}{T}\le\frac{r+s^2}{2s}.
\]
Moreover,
\[
  T\mathcal K_1(p,Y)
  \le M_T+\sum_t|p_t-q_t|,
\]
so Cauchy--Schwarz gives
\[
  \E\mathcal K_1(p,Y)
  \le\frac{r+s^2}{2s}+\sqrt r.
\]
If $r\le\eps^2/64$ and $s=\eps/2$, the right side is less than
$\eps$, while $\E R_0/T\le\eps^2$.  This contradicts
\Cref{thm:recalibration-lb} whenever $T\le c_0\eps^{-3}$.  Choosing
$\eps\asymp T^{-1/3}$ yields
\[
  \E\SwapReg_T(\ell_{\mathrm B},\cH_\star)
  =\frac12\E\overline{\SwapReg}_T
  \ge cT^{1/3}.
\]
\end{proof}

\begin{corollary}[Optimal smooth-loss time exponent from \Cref{thm:single-lipschitz} and \Cref{thm:brier-lb}]
\label{cor:optimal-smooth-time}
For the fixed pair $(\ell_{\mathrm B},\cH_\star)$,
\[
  \inf_{\text{learners}}\sup_{\text{adversaries}}
  \E\SwapReg_T(\ell_{\mathrm B},\cH_\star)
  =\Theta(T^{1/3})
\]
up to logarithmic factors and universal constants.
\end{corollary}

The half-Brier loss $\ell_{\mathrm B}(p,y)=\frac12(p-y)^2$ has
$\ell_{\mathrm B}(\cdot,0)'(p)=p$ and $\ell_{\mathrm B}(\cdot,1)'(p)=p-1$, both
bounded by $1$ in absolute value on $[0,1]$, so $\ell_{\mathrm B}$ is
$1$-Lipschitz and \Cref{thm:single-lipschitz} gives the matching upper bound
$\widetilde O(T^{1/3})$; \Cref{thm:brier-lb} gives the lower bound.

\subsection{Finite families of Lipschitz losses}

The single-loss result extends to a finite set of proper losses by running the
same second-order multicalibration algorithm on the union of the
loss-and-selector induced test classes.  The selector part contributes
$(N+1)\log|\cH|$, while the loss family contributes only $\log|\cL|$.

\begin{theorem}[Finite Lipschitz families]
\label{thm:finite-lipschitz-family}
Let $\cL$ be finite and suppose every $\ell\in\cL$ is $L$-Lipschitz.  One
learner satisfies, simultaneously for every $\ell\in\cL$,
\[
  \SwapReg_T(\ell,\cH)
  =\widetilde O\!\left(
    LT^{1/3}(\log|\cH|+\log|\cL|)^{2/3}
  \right)
\]
and
\[
  \SwapAgn_\cD(\ell,\cH)
  =\widetilde O\!\left(
    L\left(\frac{\log|\cH|+\log|\cL|}{m}\right)^{2/3}
  \right).
\]
\end{theorem}

\begin{proof}[Proof of \Cref{thm:finite-lipschitz-family}]
Run the algorithm of \Cref{thm:single-lipschitz} on the union
$\{g_{\ell,\phi}:\ell\in\cL,\phi\in\Phi_N\}$ of the selector classes, whose
logarithmic cardinality is at most $\log|\cL|+(N+1)\log|\cH|$.  Since every
$\ell\in\cL$ shares the same Lipschitz constant $L$, the coefficient
cancellation in the proof of \Cref{thm:single-lipschitz} (at $\eta=1/4$
online and $\eta=1/10$ offline) applies uniformly, and this displayed event
now bounds $\SwapReg_T(\ell,\cH)$ (respectively $\SwapAgn_\cD(\ell,\cH)$)
simultaneously for every $\ell\in\cL$.  Optimize the grid.
\end{proof}

\subsection{All Lipschitz proper losses through monotone transforms}

This subsection records a useful auxiliary route for proper losses with
Lipschitz partial losses.  The main statement is the rate theorem
\Cref{thm:offline-all-lipschitz}.  It follows by reducing the loss family to
monotone transforms in \Cref{thm:monotone-transform} and covering that
transform class in \Cref{prop:M2-cover}.

\begin{theorem}[Offline simultaneous Lipschitz proper losses]
\label{thm:offline-all-lipschitz}
Let $\cH$ be finite and let $\cL_{\mathrm{Lip}}$ be the class of proper losses
with $1$-Lipschitz partial losses.  There is a choice of prediction grid and
finite cover for the offline historical-predictor procedure such that, with
probability at least $1-\delta$ over the training sample and training
randomization,
\[
  \sup_{\ell\in\cL_{\mathrm{Lip}}}
  \SwapAgn_\cD(\ell,\cH)
  =
  \widetilde O\!\left(
    \left(\frac{\log|\cH|}{m}\right)^{2/3}
    +\frac1{\sqrt m}
    +\frac{\log(1/\delta)}{m}
  \right).
\]
\end{theorem}

The proof of \Cref{thm:offline-all-lipschitz} appears after
\Cref{thm:monotone-transform,prop:M2-cover}, which provide the reduction and
covering estimates used to obtain the displayed rate.

We now introduce the transform class that will replace the original continuum
of Lipschitz proper losses.

Let
\begin{equation}
  \cM_2
  =\left\{
    g:[0,1]\to\R:
    g(0)=0,\ g\text{ is nonincreasing, and }
    |g(p)-g(q)|\le2|p-q|
  \right\}.
  \label{eq:M2}
\end{equation}
Define
\begin{equation}
  \cG_{\mathrm{mon}}(\cH)
  =\left\{
    (x,p)\mapsto\frac{g(p)-g(\phi(p)(x))}{2}:
    g\in\cM_2,\ \phi\in\Phi_N
  \right\}.
  \label{eq:G-mon}
\end{equation}

The next theorem explains why this transform class is the right one: if the
predictor is second-order multicalibrated against these transform tests, then
it is swap-agnostic for every Lipschitz proper loss at once.

\begin{theorem}[Monotone-transform reduction]
\label{thm:monotone-transform}
If a transcript satisfies one-sided $(c,a)$ second-order multicalibration for
$\cG_{\mathrm{mon}}(\cH)$ with $c\le1/2$, then simultaneously for every
proper loss with
$1$-Lipschitz partial losses,
\begin{equation}
  \SwapReg_T(\ell,\cH)
  \le2a.
  \label{eq:monotone-transform-online}
\end{equation}
The same statement holds distributionally.
\end{theorem}

\begin{proof}[Proof of \Cref{thm:monotone-transform}]
Fix such a loss and let $g_\ell(p)=\Delta_\ell(p)-\Delta_\ell(0)$.  By the
proof of \Cref{lem:quadratic-growth} with $L=1$, $\Delta_\ell$ is nonincreasing
and $2$-Lipschitz, so $g_\ell\in\cM_2$.  By \Cref{eq:lipschitz-key-inequality}
with $L=1$, writing $a_\phi(x,p)=(g_\ell(p)-g_\ell(\phi(p)(x)))/2$ for the
corresponding test in $\cG_{\mathrm{mon}}(\cH)$,
\[
  \ell(p,y)-\ell(\phi(p)(x),y)
  \le2(y-p)a_\phi(x,p)-a_\phi(x,p)^2.
\]
Summing over the transcript and using $(c,a)$ second-order
multicalibration with $c\le1/2$,
\[
  \SwapReg_T(\ell,\phi)
  \le(2c-1)\Mass_T(a_\phi)+2a
  \le2a.
\]
Take the supremum over $\phi$.  The distributional proof is identical.
\end{proof}

To make the offline procedure finite, we approximate the transform class by a
uniform cover.  The following crude entropy estimate is enough for the rate
claimed in \Cref{thm:offline-all-lipschitz}.

\begin{proposition}[A crude cover of monotone Lipschitz transforms]
\label{prop:M2-cover}
There is a universal constant $C_0$ such that, for every
$\rho\in(0,1)$,
\[
  \log\mathcal N_\infty(\rho,\cM_2)
  \le\frac{C_0}{\rho}\log\!\left(\frac{C_0}{\rho}\right).
\]
If $\cH$ is finite, then
\[
  \mathcal N_\infty(\rho,\cG_{\mathrm{mon}}(\cH))
  \le|\cH|^{N+1}\mathcal N_\infty(\rho,\cM_2).
\]
\end{proposition}

\begin{proof}[Proof of \Cref{prop:M2-cover}]
Partition $[0,1]$ into intervals of length at most $\rho/8$.  Quantize the
value of $g$ at each grid point to a mesh of width $\rho/4$ in $[-2,0]$ and
linearly interpolate.  Lipschitzness and quantization give uniform error at
most $\rho/2$ after adjusting constants.  There are $O(1/\rho)$ grid points
and $O(1/\rho)$ possible quantized values at each point, proving the first
claim.  For the second, approximate only $g$; the same approximation works for
every difference $[g(p)-g(h(x))]/2$.
\end{proof}

\begin{remark}[Choice of cover radius]
Because the final reduction sums over $N+1$ buckets, a uniform cover radius of
order $\eps$ contributes order $\eps$ directly in the offset term.  For smooth
transform families, the logarithm of the cover size often grows only
logarithmically in $1/\rho$.  For the class of all monotone Lipschitz
transforms, a crude uniform cover is larger; sharper distributional covers are
an important direction for improvement.
\end{remark}

We now combine the reduction and the cover to prove the rate theorem stated
at the beginning of the subsection.

\begin{proof}[Proof of \Cref{thm:offline-all-lipschitz}]
Let $\widetilde\cG$ be a uniform $\rho$-cover of
$\cG_{\mathrm{mon}}(\cH)$ and run the offline algorithm of
\Cref{thm:online-to-batch} on $\widetilde\cG$ with rate $\eta=1/20$.  This
gives distributional $(5\eta,\bar a_m)$ second-order multicalibration for
$\widetilde\cG$ with
\[
  \bar a_m
  =
  \frac{
    2\log(12|\widetilde\cG|/\delta)+2m/N^2
  }{\eta m}.
\]
By \Cref{lem:cover-transfer}, the output satisfies distributional
$(\bar c_m,\bar a_m')$ second-order multicalibration for
$\cG_{\mathrm{mon}}(\cH)$ with
\[
  \bar c_m=10\eta=\frac12,
  \qquad
  \bar a_m'
  =\bar a_m+(10\eta+1)\rho
  =\frac{\Lambda_m}{\eta m}+(10\eta+1)\rho,
\]
where
\[
  \Lambda_m
  =
  2\log\!\left(\frac{12|\widetilde\cG|}{\delta}\right)
  +\frac{2m}{N^2}.
\]
\Cref{thm:monotone-transform} therefore gives, simultaneously for all
proper losses with $1$-Lipschitz partial losses,
\[
  \SwapAgn_\cD(\ell,\cH)
  \le2\bar a_m'.
\]
Using \Cref{prop:M2-cover},
\[
  \log|\widetilde\cG|
  \le (N+1)\log|\cH|
  +O\!\left(\frac1\rho\log\frac1\rho\right).
\]
Optimizing the grid size gives
$N\asymp(m/\log|\cH|)^{1/3}$, contributing
$\widetilde O((\log|\cH|/m)^{2/3})$.  Optimizing the transform-cover radius
gives $\rho\asymp m^{-1/2}$ up to logarithmic factors, contributing
$\widetilde O(m^{-1/2})$.  The remaining confidence term contributes
$O(\log(1/\delta)/m)$.
\end{proof}

\section{Bounded Proper Losses}
\label{sec:bounded}

A bounded proper loss need not be Lipschitz (equivalently, by
\Cref{cor:lipschitz-smooth-equivalence}, its negative Bayes risk need not be
smooth): its slope may jump, so \Cref{lem:quadratic-growth} need not hold for
any finite constant, and the single-loss and finite-family arguments of
\Cref{sec:proper-families} do not apply.  Worse, even the trick that handled
a continuum of Lipschitz losses -- covering the loss family in the sup norm
and running the finite-family algorithm on the cover
(\Cref{thm:offline-all-lipschitz}) -- breaks down here.  A threshold is a
discontinuity, not a Lipschitz parameter: two V-shaped losses with
arbitrarily close thresholds disagree by a constant amount at every point
strictly between the two thresholds, so the V-shaped family admits no finite
cover of this kind at all.  This is why passing from finitely many
thresholds to every threshold in $[0,1]$ needs a genuinely new ingredient --
a finite-range assumption on the comparator class $\cH$, in place of an
approximation argument -- and why this section has more moving parts than
the Lipschitz and convex-Lipschitz families treated elsewhere in the paper.

We record the known V-shaped basis decomposition and state our main
bounded-loss rate, \Cref{thm:bounded-proper-rates}, in
\Cref{subsec:v-decomposition}.  \Cref{subsec:v-finite} proves the rate for
any finite set of thresholds by applying the square-root form of
\Cref{alg:approachability} to V-shaped basis losses.
\Cref{subsec:v-full-range} is where the new ingredient appears: it extends
the finite-threshold guarantee to \emph{every} threshold in $[0,1]$, exactly
and with no approximation, under the finite-range assumption.
\Cref{subsec:v-assembly} integrates this over the representation of
\Cref{lem:proper-measure-representation} to complete the proof, and
\Cref{subsec:v-lower} gives matching lower bounds.

\subsection{The V-shaped decomposition and the main theorem}
\label{subsec:v-decomposition}

For $v\in[0,1]$ and $\tau\in[-1,1]$, define
\begin{equation}
  \Th_{v,\tau}(u)
  =\begin{cases}
    1,&u<v,\\
    \tau,&u=v,\\
    -1,&u>v,
  \end{cases}
  \qquad
  \ell_{v,\tau}(u,y)=(y-v)\Th_{v,\tau}(u).
  \label{eq:v-primitive}
\end{equation}
The tie parameter is needed only to represent the value of a proper loss at a
kink of its Bayes risk.  The canonical choice $\tau=1$ recovers the convention
$2\1\{u\le v\}-1$.

\begin{lemma}[Measure representation of bounded proper losses]
\label{lem:proper-measure-representation}
Let $\ell$ be proper and satisfy $|\ell(p,y)|\le M$.  Then there are
constants $c,d\in\R$, a finite nonnegative Borel measure $\mu_\ell$ on
$[0,1]$, and tie parameters $\tau_v\in[-1,1]$ for the atoms of $\mu_\ell$
such that
\begin{equation}
  \ell(p,y)
  =c+dy+\int_{[0,1]}\ell_{v,\tau_v}(p,y)\,d\mu_\ell(v).
  \label{eq:proper-measure-representation}
\end{equation}
Moreover,
\[
  \mu_\ell([0,1])\le2M.
\]
The values of $\tau_v$ on the nonatomic part of $\mu_\ell$ are irrelevant.
This V-shaped threshold decomposition of bounded proper losses is due to
\citet{kleinberg2023ucalibration} and \citet{li2022optimization}.
\end{lemma}

\begin{theorem}[Bounded proper losses]
\label{thm:bounded-proper-rates}
Suppose there is a finite set $\Gamma_\cH\subseteq[0,1]$ with
$h(\cX)\subseteq\Gamma_\cH$ for every $h\in\cH$, and write
$K=|\Gamma_\cH|$.  The localized algorithm based on
\Cref{alg:approachability} has a choice of prediction grid such that, with
probability at least $1-\delta$, it is simultaneously swap-agnostic for all
proper losses bounded in $[-1,1]$:
\[
  \sup_{\ell:\ |\ell|\le1}
  \SwapReg_T(\ell,\cH)
  =
  \widetilde O\!\left(
    \sqrt{T\left(\log|\cH|+\log K+\log\frac1\delta\right)}
  \right).
\]
The corresponding historical-predictor conversion satisfies
\[
  \sup_{\ell:\ |\ell|\le1}
  \SwapAgn_\cD(\ell,\cH)
  =
  \widetilde O\!\left(
    \sqrt{\frac{\log|\cH|+\log K+\log(1/\delta)}m}
  \right).
\]
\end{theorem}

The proof of \Cref{lem:proper-measure-representation} is given in
\Cref{subsec:v-finite} below, alongside the basis-loss drift calculation;
the proof of \Cref{thm:bounded-proper-rates} itself is assembled in
\Cref{subsec:v-assembly}.

\subsection{Controlling a finite set of thresholds}
\label{subsec:v-finite}

The proof starts with a drift inequality for a single V-shaped basis loss.
This is the main piece of machinery that lets the localized algorithm handle
nonsmooth losses.

\begin{lemma}[V-shaped basis-loss drift]
\label{lem:v-drift}
The loss $\ell_{v,\tau}$ is proper.  For $p,q\in[0,1]$, define
\[
  g_{v,\tau}(p,q)
  =\frac12\left(
    \Th_{v,\tau}(p)-\Th_{v,\tau}(q)
  \right)\in[-1,1].
\]
Then, for every $y\in\{0,1\}$,
\begin{equation}
  \ell_{v,\tau}(p,y)-\ell_{v,\tau}(q,y)
  \le2(y-p)g_{v,\tau}(p,q)
    -2|p-v|g_{v,\tau}(p,q)^2.
  \label{eq:v-drift}
\end{equation}
If neither $p$ nor $q$ equals $v$, equality holds.
\end{lemma}

\begin{proof}[Proof of \Cref{lem:v-drift}]
Under a true mean $p$, the expected basis loss at prediction $q$ is
$(p-v)\Th_{v,\tau}(q)$.  If $p<v$, it is minimized by predictions below $v$;
if $p>v$, it is minimized by predictions above $v$; and if $p=v$, every
prediction is optimal.  Thus the loss is proper.

The exact algebra gives
\begin{align*}
  \ell_{v,\tau}(p,y)-\ell_{v,\tau}(q,y)
  &=2(y-v)g_{v,\tau}(p,q)\\
  &=2(y-p)g_{v,\tau}(p,q)
    +2(p-v)g_{v,\tau}(p,q).
\end{align*}
If $p\neq v$ and $g_{v,\tau}(p,q)\neq0$, the sign of
$g_{v,\tau}(p,q)$ is opposite to the sign of $p-v$.  If $p=v$, the
product below is zero.  In either case,
\[
  (p-v)g_{v,\tau}(p,q)
  \le-|p-v|\,|g_{v,\tau}(p,q)|
  \le-|p-v|g_{v,\tau}(p,q)^2.
\]
If neither prediction is tied at $v$, then $|g|\in\{0,1\}$ and equality holds.
\end{proof}

We can now prove the basis decomposition of
\Cref{lem:proper-measure-representation} above.

\begin{proof}[Proof of \Cref{lem:proper-measure-representation}]
Let $F=F_\ell$.  It is convex, and
$-\Delta_\ell(p)\in\partial F(p)$.  Let $D^2F$ be the distributional second
derivative, a finite nonnegative measure, and set
\[
  \mu_\ell=\frac12D^2F.
\]
The standard one-dimensional representation of a convex function gives
\begin{equation}
  F(p)=\alpha+\beta p+
  \int_{[0,1]}|p-v|\,d\mu_\ell(v).
  \label{eq:convex-abs-rep}
\end{equation}
At every atom $v$, choose $\tau_v$ so that the selected subgradient of the
term $\mu_\ell(\{v\})|p-v|$ at $p=v$ equals
$-\mu_\ell(\{v\})\tau_v$.  This is possible because the subdifferential of
$|p-v|$ at $v$ is $[-1,1]$.  With these choices, the selected subgradient of
\eqref{eq:convex-abs-rep} is exactly $-\Delta_\ell(p)$.

The proper loss generated by a convex potential $F$ and selected subgradient
$s(p)\in\partial F(p)$ is
\[
  \ell(p,y)=-F(p)-(y-p)s(p).
\]
For $F_v(p)=|p-v|$ and $s_v(p)=-\Th_{v,\tau_v}(p)$, this generated loss is
exactly $\ell_{v,\tau_v}(p,y)$.  The affine part
$\alpha+\beta p$ contributes the prediction-independent term
$-\alpha-\beta y$.  This proves \eqref{eq:proper-measure-representation}.

Finally, $F'=-\Delta_\ell$ wherever the derivative exists.  Since
$|\ell(p,y)|\le M$, $|\Delta_\ell(p)|\le2M$.  Therefore the total increase of
the monotone derivative $F'$ is at most $4M$, and
\[
  \mu_\ell([0,1])=\frac12D^2F([0,1])\le2M.
\]
\end{proof}

\begin{remark}[Extreme tie values suffice]
\label{rem:extreme-ties}
For every $\tau\in[-1,1]$, $\ell_{v,\tau}=\tfrac{1+\tau}2\ell_{v,1}
+\tfrac{1-\tau}2\ell_{v,-1}$ pointwise, since the two sides agree away from
$u=v$ and, at $u=v$, evaluate to
$\tfrac{1+\tau}2(1)+\tfrac{1-\tau}2(-1)=\tau$.  Splitting each atom of
$\mu_\ell$ at $v$ with mass $w$ and tie $\tau_v$ into mass $w(1+\tau_v)/2$ at
$(v,1)$ and mass $w(1-\tau_v)/2$ at $(v,-1)$ therefore refines the
representation of \Cref{lem:proper-measure-representation} to one using
only the two extreme tie values, without changing the total mass
$\mu_\ell([0,1])$.  We use only $\tau\in\{-1,1\}$ from here on.
\end{remark}

The finite-family theorem will sum the basis-loss bounds over grid points.  The
only extra estimate needed for that summation is the following harmonic
inequality.

\begin{lemma}[Harmonic drift--variance inequality]
\label{lem:harmonic-drift}
Fix $v\in[0,1]$, set $a_i=|\gam_i-v|$, and let
$D_0,\ldots,D_N\ge0$ satisfy $\sum_iD_i\le T$.  For every $R>0$,
\begin{equation}
  \sum_{i=0}^N\left(R\sqrt{D_i}-2a_iD_i\right)
  \le R\sqrt{2T}+\frac{R^2}{4}NH_N,
  \label{eq:harmonic-drift}
\end{equation}
where $H_N=\sum_{k=1}^N1/k$.
\end{lemma}

\begin{proof}[Proof of \Cref{lem:harmonic-drift}]
The summand is potentially large only when the grid point $\gam_i$ lies close
to the threshold $v$, because then the negative drift coefficient $a_i$ is
small.  We first isolate those nearby grid points.  There are at most two grid
points with $a_i<1/N$, one on each side of $v$, and their contribution is at
most
\[
  R\sum_{i:a_i<1/N}\sqrt{D_i}
  \le R\sqrt{2\sum_iD_i}
  \le R\sqrt{2T}
\]
by Cauchy--Schwarz.

For every remaining point, the drift term controls large values of $D_i$.
Writing $z=\sqrt D$, the expression
$R\sqrt D-2a_iD=Rz-2a_iz^2$ is a concave quadratic in $z\ge0$.  Its maximum
occurs at $z=R/(4a_i)$, which gives
\[
  R\sqrt D-2a_iD\le\frac{R^2}{8a_i}.
\]

It remains to sum the reciprocal distances.  Away from the two nearest grid
points, the grid spacing implies that on each side of $v$ the distances are
at least $1/N,2/N,\ldots,N/N$, up to constants.  Hence the reciprocal
distances are bounded by a harmonic sum:
\[
  \sum_{i:a_i\ge1/N}\frac1{a_i}\le2NH_N.
\]
Adding the nearby contribution and the far-grid contribution proves the
claim.
\end{proof}

We next state the finite-threshold online theorem that drives the bounded-loss
rate.  Let $\cV$ be a finite set of pairs $(v,\tau)$ and define
\[
  g_{v,\tau,h}(x,p)
  =\frac12\left(
    \Th_{v,\tau}(p)-\Th_{v,\tau}(h(x))
  \right).
\]

\begin{theorem}[Finite V-shaped family]
\label{thm:finite-v-family}
Let $\cH$ and $\cV$ be finite.  Run \Cref{alg:approachability} on the
localized test class
\[
  \cG_{\cV}
  =\left\{
    (x,p)\mapsto
    \1\{p=\gam_i\}g_{v,\tau,h}(x,\gam_i):
    i=0,\ldots,N,\ (v,\tau)\in\cV,\ h\in\cH
  \right\}
\]
using the square-root form \Cref{cor:sqrt-somc}.
There is a universal constant $C'$ such that, with probability at least
$1-\delta$, simultaneously for every $(v,\tau)\in\cV$ and every swap
comparator,
\begin{align}
  \SwapReg_T(\ell_{v,\tau},\cH)
  &\le C'\sqrt{\Lambda_TT}
    +C'\Lambda_TNH_N
    +C'\Lambda_T(N+1),
  \label{eq:v-family-regret}
\end{align}
where $\Lambda_T=\Lambda_T(\cG_{\cV},\delta,N)$ is as in
\Cref{cor:sqrt-somc}.
\end{theorem}

\begin{proof}[Proof of \Cref{thm:finite-v-family}]
On the simultaneous second-order event, fix $(v,\tau)$ and a swap rule
$h_i$.  Put
\[
  D_i=\sum_{t\in I_i}g_{v,\tau,h_i}(x_t,\gam_i)^2.
\]
For each bucket $i$, the test $(x,p)\mapsto\1\{p=\gam_i\}g_{v,\tau,h_i}(x,p)$
lies in $\cG_\cV$, and its online bias is exactly
$\sum_{t\in I_i}(y_t-\gam_i)g_{v,\tau,h_i}(x_t,\gam_i)$, since $p_t=\gam_i$
for $t\in I_i$.  By \Cref{lem:v-drift} and \Cref{cor:sqrt-somc} applied to
this test,
\begin{align*}
  \SwapReg_T(\ell_{v,\tau},\phi)
  &\le\sum_i\left(
    2C\sqrt{\Lambda_TD_i}
    -2|\gam_i-v|D_i
  \right)+2C\Lambda_T(N+1),
\end{align*}
where $C$ is the constant of \Cref{cor:sqrt-somc}.  Since the buckets are
disjoint and $|g|\le1$, $\sum_iD_i\le T$.  Apply \Cref{lem:harmonic-drift}
with $R=2C\sqrt{\Lambda_T}$ to bound the sum by
$2C\sqrt{2\Lambda_TT}+C^2\Lambda_TNH_N$, which proves
\eqref{eq:v-family-regret} for a suitable universal $C'$.
\end{proof}

\subsection{From finitely many thresholds to every threshold}
\label{subsec:v-full-range}

\Cref{thm:finite-v-family} controls only the finitely many thresholds in
$\cV$, while \Cref{lem:proper-measure-representation} represents a bounded
proper loss using a threshold measure supported on all of $[0,1]$.  Bridging
this gap by approximating the continuum of thresholds with a uniform
$\rho$-cover of the \emph{loss} $\ell_{v,\tau}$ does not work: for any
$v\ne v'$, taking $p$ strictly between them gives
$\Th_{v,\tau}(p)\ne\Th_{v',\tau'}(p)$, so
$\sup_{p,y}|\ell_{v,\tau}(p,y)-\ell_{v',\tau'}(p,y)|$ is bounded away from
$0$ however close $v$ and $v'$ are -- the threshold is a discontinuity in
$v$, not a Lipschitz parameter, so the V-shaped family admits no finite
$\rho$-cover in this sense for any $\rho<2$.  Instead, we use the following
assumption to control every threshold \emph{exactly}, with no approximation
at all.

Only in this section, assume that there is a finite set
$\Gamma_\cH\subseteq[0,1]$ with $h(\cX)\subseteq\Gamma_\cH$ for every
$h\in\cH$, and write $K=|\Gamma_\cH|$.  Unlike the prediction grid
$\Gamma_N$, which is a modeling choice of the learner, $\Gamma_\cH$
constrains the comparator class itself; it is natural for quantized score
classes but excludes classes with a continuum of possible outputs.  Let
$\cS_N=\Gamma_N\cup\Gamma_\cH$.

\begin{corollary}[Every threshold from a finite range]
\label{cor:v-family-full-range}
Under the finite-range assumption above, run \Cref{thm:finite-v-family} with
$\cV=\cS_N\times\{-1,1\}$, so $|\cV|\le2(N+1+K)$.  Then, on the resulting
event of probability at least $1-\delta$,
\eqref{eq:v-family-regret} (with $\cV=\cS_N\times\{-1,1\}$ in the definition
of $\Lambda_T$) holds simultaneously for \emph{every} $v\in[0,1]$,
$\tau\in\{-1,1\}$, and every swap comparator -- not merely the pairs already
in $\cV$.
\end{corollary}

\begin{proof}[Proof of \Cref{cor:v-family-full-range}]
Fix $v\in[0,1]$, $\tau\in\{-1,1\}$, and a swap rule $\phi$ with
$h_i=\phi(\gam_i)$.  Let $v_{\mathrm{rep}}$ be the smallest element of
$\cS_N$ with $v_{\mathrm{rep}}\ge v$ (taking $v_{\mathrm{rep}}=1$ if
$v>\max\cS_N$), and let $\tau_{\mathrm{rep}}=\tau$ if $v_{\mathrm{rep}}=v$
and $\tau_{\mathrm{rep}}=-1$ otherwise.  For every $z\in\cS_N$,
$\Th_{v_{\mathrm{rep}},\tau_{\mathrm{rep}}}(z)=\Th_{v,\tau}(z)$: if
$v_{\mathrm{rep}}=v$ this is immediate, and otherwise every $z\in\cS_N$
satisfies $z<v_{\mathrm{rep}}$ or $z>v_{\mathrm{rep}}$ (there is no point of
$\cS_N$ strictly between $v$ and $v_{\mathrm{rep}}$, by minimality of
$v_{\mathrm{rep}}$), so $z<v_{\mathrm{rep}}\Leftrightarrow z\le v
\Leftrightarrow z<v$ (using $z\ne v$ since $v_{\mathrm{rep}}\ne v$ means
$v\notin\cS_N$) gives $\Th_{v_{\mathrm{rep}},-1}(z)=\Th_{v,\tau}(z)$ in both
branches.

Since $\gam_i\in\Gamma_N\subseteq\cS_N$ and, by the finite-range assumption,
$h_i(x_t)\in\Gamma_\cH\subseteq\cS_N$ for every $t$, this gives
$g_{v,\tau,h_i}(x_t,\gam_i)=g_{v_{\mathrm{rep}},\tau_{\mathrm{rep}},h_i}
(x_t,\gam_i)$ for every $t$, so the two tests have identical realized bias
and mass on this transcript.  Since $(v_{\mathrm{rep}},\tau_{\mathrm{rep}})
\in\cV$, the guarantee of \Cref{thm:finite-v-family} for
$(v_{\mathrm{rep}},\tau_{\mathrm{rep}})$ therefore controls
$D_i=\sum_{t\in I_i}g_{v,\tau,h_i}(x_t,\gam_i)^2$ exactly as in that proof.
The remainder of the argument is unchanged: the exact drift identity of
\Cref{lem:v-drift} for the \emph{true} $v$ (not $v_{\mathrm{rep}}$) gives
\[
  \SwapReg_T(\ell_{v,\tau},\phi)
  \le\sum_i\left(2C\sqrt{\Lambda_TD_i}-2|\gam_i-v|D_i\right)
  +2C\Lambda_T(N+1),
\]
and \Cref{lem:harmonic-drift} applies verbatim for this arbitrary
$v\in[0,1]$ (it was already stated for arbitrary $v\in[0,1]$, not merely
$v\in\cS_N$), proving \eqref{eq:v-family-regret} for every $v,\tau,\phi$
simultaneously on the same event.
\end{proof}

\subsection{Assembling the bound}
\label{subsec:v-assembly}

Because bounded proper losses are positive mixtures of V-shaped basis losses,
the finite-threshold theorem transfers linearly to finite mixtures.

\begin{corollary}[Positive mixtures of V-shaped basis losses from \Cref{thm:finite-v-family}]
\label{cor:v-mixtures}
Suppose
\[
  \ell(p,y)=c+dy+\sum_{j=1}^Jw_j\ell_{v_j,\tau_j}(p,y),
  \qquad w_j\ge0,\quad\sum_jw_j\le W.
\]
On the event of \Cref{thm:finite-v-family},
\[
  \SwapReg_T(\ell,\cH)
  \le W\left(
    C'\sqrt{\Lambda_TT}
    +C'\Lambda_TNH_N
    +C'\Lambda_T(N+1)
  \right).
\]
The same conclusion holds for a positive Borel mixture if second-order
multicalibration is known simultaneously for every basis loss in its support.
\end{corollary}

\begin{proof}[Proof of \Cref{cor:v-mixtures}]
The term $c+dy$ cancels from regret.  For each fixed swap rule, integrate the
basis-loss regret inequalities against the nonnegative weights and then take
the supremum over swap rules.
\end{proof}

A uniform cover of the \emph{loss} family, rather than of the test class, is
also a useful tool -- we record it here because it will be applied to a
genuinely continuously-parameterized loss family in
\Cref{sec:convex-lipschitz}, unlike the V-shaped family above.

\begin{proposition}[Uniform loss-cover transfer]
\label{prop:loss-cover-transfer}
Suppose $\widetilde\cL$ is a uniform $\rho$-cover of a loss family $\cL$.  If
one predictor satisfies
\[
  \sup_{\widetilde\ell\in\widetilde\cL}
  \SwapAgn_\cD(\widetilde\ell,\cH)\le R,
\]
then
\[
  \sup_{\ell\in\cL}\SwapAgn_\cD(\ell,\cH)\le R+2\rho.
\]
The corresponding online statement has additive error $2T\rho$.
\end{proposition}

\begin{proof}[Proof of \Cref{prop:loss-cover-transfer}]
For any prediction and any comparator prediction, replacing $\ell$ by a
uniformly $\rho$-close loss changes their loss difference by at most
$2\rho$ per example.
\end{proof}

We now combine the finite basis theorem, its full-range extension, and the
positive-mixture transfer to prove the main bounded-loss theorem.

\begin{proof}[Proof of \Cref{thm:bounded-proper-rates}]
By \Cref{cor:v-family-full-range}, running \Cref{thm:finite-v-family} on
$\cV=\cS_N\times\{-1,1\}$ (of size at most $2(N+1+K)$) controls
$\SwapReg_T(\ell_{v,\tau},\cH)$ simultaneously for \emph{every}
$v\in[0,1]$, $\tau\in\{-1,1\}$, and swap comparator, on one event of
probability at least $1-\delta$.  By \Cref{rem:extreme-ties}, every proper
loss bounded in $[-1,1]$ is, up to the prediction-independent affine term
(which cancels from swap regret), a positive Borel mixture with total mass
at most $2$ of basis losses $\ell_{v,\tau}$ with $\tau\in\{-1,1\}$.  Since
second-order multicalibration is known simultaneously for every basis loss
in this support, \Cref{cor:v-mixtures} applies directly, with no
approximation of the threshold continuum needed, giving the displayed
uniform bound after optimizing the grid at the scale
$N\asymp\sqrt{T/\Lambda_T}$ and absorbing the harmonic and $\log K$ factors
into $\widetilde O(\cdot)$.

For the offline statement, run the same localized signed test class in the
historical-predictor conversion.  The distributional square-root bound
\Cref{cor:offline-sqrt-somc}, combined with the same exact-threshold-matching
argument as in the proof of \Cref{cor:v-family-full-range} (the finite-range
assumption places every realized $\gam_i$ and $h_i(x_t)$ in $\cS_N$ exactly
as before) and \Cref{cor:v-mixtures}, gives the same calculation with $T$
replaced by $m$ and regret normalized by $m$.
\end{proof}

\begin{remark}[The finite-range assumption is substantive]
Unlike the approximation arguments used elsewhere in this paper, the
finite-range assumption of \Cref{thm:bounded-proper-rates} is not a
without-loss-of-generality discretization of an arbitrary real-valued
hypothesis class: a bounded proper loss may jump at a threshold, so the
naive fix -- approximating the threshold continuum by a uniform cover of the
V-shaped \emph{losses} -- fails outright (as discussed before
\Cref{cor:v-family-full-range}), and we do not know a substitute argument
that removes the assumption while keeping an exact (rather than merely
approximate) threshold guarantee.  The assumption is natural for quantized
score classes, and it exactly controls the number of threshold patterns
induced by the comparator; removing it, or replacing it with a
distribution- or sample-dependent complexity condition on the induced
threshold classes, is an interesting direction for improvement.
\end{remark}

\subsection{Lower bounds}
\label{subsec:v-lower}

The square-root rates in \Cref{thm:bounded-proper-rates} are also unavoidable
for broad bounded proper loss families.  The lower bounds already hold for one
V-shaped proper loss.

\begin{lemma}[V-shaped online lower bound]
\label{lem:v-online-lb}
Let $1\le d\le T/2$.  There are a domain and a class $\cH$ with
$|\cH|=2^d$ such that every randomized learner admits a deterministic sequence
with
\[
  \E\SwapReg_T(\ell_\vee,\cH)
  \ge\frac1{48}\sqrt{dT}
  =\Omega\!\left(\sqrt{T\log|\cH|}\right),
\]
where
\[
  \ell_\vee(p,y)
  =
  \begin{cases}
    y-\frac12,&p\le\frac12,\\
    \frac12-y,&p>\frac12.
  \end{cases}
\]
\end{lemma}

\begin{proof}[Proof of \Cref{lem:v-online-lb}]
Take $\cX=[d]$ and $\cH=\{0,1\}^{[d]}$.  Partition the rounds into
$d$ blocks $I_j$ of sizes differing by at most one and set $x_t=j$ on
$I_j$.  Draw independent fair labels.  For any predictable learner
prediction, the conditional expected V-shaped loss is zero.  Let
\[
  S_j=\sum_{t\in I_j}(2Y_t-1).
\]
The cumulative losses of predictions $0$ and $1$ on block $j$ are
$S_j/2$ and $-S_j/2$, respectively.  Hence the best hypothesis has loss
$-\frac12\sum_j|S_j|$, and swap regret dominates external regret.  For a
length-$n$ Rademacher sum $S$, Paley--Zygmund applied to $S^2$ gives
$\E|S|\ge\sqrt n/(12\sqrt2)$.  Since every block has size at least
$T/(2d)$,
\[
  \E\SwapReg_T(\ell_\vee,\cH)
  \ge\frac12\sum_j\E|S_j|
  \ge\frac1{48}\sqrt{dT}.
\]
Averaging over labels gives a deterministic realization with the same expected
regret over the learner's randomness.
\end{proof}

\begin{lemma}[V-shaped offline lower bound]
\label{lem:v-offline-lb}
For every $1\le d\le m$ there are a domain, a class $\cH$ with
$|\cH|=2^d$, and a family of distributions such that every possibly
randomized learner based on $m$ i.i.d. samples satisfies
\[
  \sup_\cD\E\SwapAgn_\cD(\ell_\vee,\cH)
  \ge c\sqrt{\frac d m}
  =\Omega\!\left(\sqrt{\frac{\log|\cH|}{m}}\right).
\]
\end{lemma}

\begin{proof}[Proof of \Cref{lem:v-offline-lb}]
Take $\cX=[d]$ and $\cH=\{0,1\}^{[d]}$.  For
$\sigma\in\{-1,+1\}^d$, let $X$ be uniform on $[d]$ and
\[
  Y\mid X=j\sim\Ber\left(\frac12+\alpha\sigma_j\right),
  \qquad
  \alpha=c_0\sqrt{\frac d m},
\]
with $c_0$ small enough that $\alpha\le1/4$.  Under $\ell_\vee$, only
the side of $1/2$ on which the learner predicts matters, and choosing the
wrong side on coordinate $j$ costs conditional excess $2\alpha$.  Swap
excess dominates external excess to the best binary hypothesis.  Assouad's
lemma gives a universal lower bound on the expected fraction of incorrectly
recovered signs because the KL divergence between neighboring parameters is at
most $C m\alpha^2/d=O(1)$.  Therefore the expected excess is at least
$c\alpha$, proving the claim.
\end{proof}

\section{Convex $1$-Lipschitz Proper Losses}
\label{sec:convex-lipschitz}

This section gives the second basis reduction.  The first main statement is
structural: every proper loss with convex $1$-Lipschitz partial losses is a
positive mixture of clipped-ReLU proper losses, up to an affine term depending
only on the outcome \Cref{prop:clipped-decomposition-front}.  The second main
statement is algorithmic: this
decomposition yields simultaneous online swap regret and offline
swap-agnostic learning guarantees for the whole convex Lipschitz family
\Cref{thm:offline-convex-lipschitz}.  The proof proceeds through the primitive
identity \Cref{lem:clipped-primitive}, the BV curvature lemma
\Cref{lem:curvature-bv}, the interval layer-cake lemma
\Cref{lem:interval-layer-cake}, and the finite cover in
\Cref{prop:clip-cover}.

For $0\le a\le b\le1$, define the clipped-ReLU feature
\begin{equation}
  \varphi_{a,b}(z)
  =(z-a)_+-(z-b)_+,
  \label{eq:clipped-relu}
\end{equation}
and the convex potential
\begin{equation}
  F_{a,b}(p)
  =\frac12(p-a)_+^2-\frac12(p-b)_+^2.
  \label{eq:clipped-potential}
\end{equation}
Then
\[
  F_{a,b}'(p)=\varphi_{a,b}(p),
  \qquad
  F_{a,b}''(p)=\1\{a<p<b\}
\]
almost everywhere.  Define the corresponding proper loss
\begin{equation}
  \ell_{a,b}(p,y)
  =-F_{a,b}(p)-(y-p)\varphi_{a,b}(p).
  \label{eq:clipped-loss}
\end{equation}

\begin{proposition}[Clipped-ReLU decomposition of convex Lipschitz proper losses]
\label{prop:clipped-decomposition-front}
For each $0\le a\le b\le1$, the loss $\ell_{a,b}$ is proper.  Moreover, if
$\ell$ is proper and both partial losses are convex and $1$-Lipschitz, then
there are constants $c,d\in\R$ and a finite nonnegative Borel measure
$\nu_\ell$ on $\cI=\{(a,b):0\le a\le b\le1\}$ such that
\begin{equation}
  \ell(p,y)
  =c+dy+\int_{\cI}\ell_{a,b}(p,y)\,d\nu_\ell(a,b),
  \label{eq:clipped-decomposition-front}
\end{equation}
and $\nu_\ell(\cI)\le5$.
\end{proposition}

The proposition reduces the loss family to a continuum of interval-indexed
primitive losses.  To obtain an offline learner, we control a finite cover of
the corresponding clipped-ReLU event class and then integrate the primitive
swap-agnostic bounds using the positive measure in the decomposition.

\begin{theorem}[Convex Lipschitz proper losses]
\label{thm:offline-convex-lipschitz}
Let $\cH$ be finite and let $\cL_{\mathrm{cvx}}$ denote the class of proper
losses with convex $1$-Lipschitz partial losses.  There are choices of
prediction grid and finite cover for the online algorithm and for the offline
historical-predictor procedure such that, with probability at least
$1-\delta$,
\[
  \sup_{\ell\in\cL_{\mathrm{cvx}}}
  \SwapReg_T(\ell,\cH)
  =\widetilde O\!\left(
    T^{1/3}(\log|\cH|)^{2/3}
    +\log\frac1\delta
  \right)
  \quad\text{online, and}
\]
with probability at least $1-\delta$ over the training sample and training
randomization,
\[
  \sup_{\ell\in\cL_{\mathrm{cvx}}}
  \SwapAgn_\cD(\ell,\cH)
  =\widetilde O\!\left(
    \left(\frac{\log|\cH|}{m}\right)^{2/3}
    +\frac{\log(1/\delta)}{m}
  \right)
  \quad\text{offline.}
\]
\end{theorem}

\begin{lemma}[Clipped-ReLU primitive]
\label{lem:clipped-primitive}
The loss $\ell_{a,b}$ is proper.  For every $p,q\in[0,1]$ and
$y\in\{0,1\}$,
\begin{align}
  \ell_{a,b}(p,y)-\ell_{a,b}(q,y)
  &=(y-p)[\varphi_{a,b}(q)-\varphi_{a,b}(p)]
    -D_{F_{a,b}}(p,q)
  \label{eq:clipped-exact}\\
  &\le(y-p)[\varphi_{a,b}(q)-\varphi_{a,b}(p)]
    -\frac12[\varphi_{a,b}(q)-\varphi_{a,b}(p)]^2.
  \label{eq:clipped-quadratic}
\end{align}
\end{lemma}

\begin{proof}[Proof of \Cref{lem:clipped-primitive}]
Under true mean $p$, the expected loss at prediction $q$ is
\[
  -F_{a,b}(q)-(p-q)F_{a,b}'(q).
\]
Subtracting the truthful value $-F_{a,b}(p)$ gives
$D_{F_{a,b}}(p,q)\ge0$, proving properness and
\eqref{eq:clipped-exact}.  Since $F_{a,b}'=\varphi_{a,b}$ is
$1$-Lipschitz, \Cref{lem:smooth-bregman} gives
\[
  D_{F_{a,b}}(p,q)
  \ge\frac12[\varphi_{a,b}(p)-\varphi_{a,b}(q)]^2,
\]
which proves \eqref{eq:clipped-quadratic}.
\end{proof}

We next prove a positive representation theorem.  The key fact is that
convexity of both partial losses forces the curvature density of the negative
Bayes risk to have bounded variation.

\begin{lemma}[Curvature of a convex Lipschitz proper loss]
\label{lem:curvature-bv}
Suppose $\ell$ is proper and both partial losses are convex and
$1$-Lipschitz.  Then $F_\ell$ is continuously differentiable,
$F_\ell'$ is $2$-Lipschitz, and there is a nonnegative function
$w_\ell\in BV([0,1])$ such that
\[
  F_\ell''(p)=w_\ell(p)
\]
almost everywhere and
\begin{equation}
  \|w_\ell\|_\infty\le2,
  \qquad
  \TV(w_\ell)\le6.
  \label{eq:curvature-bv}
\end{equation}
\end{lemma}

\begin{proof}[Proof of \Cref{lem:curvature-bv}]
By the proof of \Cref{lem:quadratic-growth} with $L=1$, $F_\ell'$ exists and is
$2$-Lipschitz.  Hence it is absolutely continuous, with a derivative
$w_\ell=F_\ell''$ almost everywhere.  Convexity of $F_\ell$ and the
Lipschitz bound give $0\le w_\ell\le2$.

The canonical formulas for a proper loss generated by $F_\ell$ are
\[
  \ell_0(p)=-F_\ell(p)+pF_\ell'(p),
  \qquad
  \ell_1(p)=-F_\ell(p)-(1-p)F_\ell'(p).
\]
Differentiating almost everywhere gives
\[
  \ell_0'(p)=p w_\ell(p),
  \qquad
  \ell_1'(p)=-(1-p)w_\ell(p).
\]
Since $\ell_0$ is convex, $u(p)=pw_\ell(p)$ has a nondecreasing
representative; since $\ell_1$ is convex,
$v(p)=(1-p)w_\ell(p)$ has a nonincreasing representative.  The
$1$-Lipschitz assumption implies $0\le u,v\le1$.

On $[0,1/2]$,
\[
  w_\ell(p)=\frac{v(p)}{1-p}.
\]
The BV product inequality yields
\[
  \TV_{[0,1/2]}(w_\ell)
  \le\left\|\frac1{1-p}\right\|_\infty\TV(v)
    +\|v\|_\infty\TV\!\left(\frac1{1-p}\right)
  \le2+1=3.
\]
On $[1/2,1]$, write $w_\ell(p)=u(p)/p$ to obtain the same bound.  The
one-sided limits at $1/2$ agree: monotonicity of $u$ gives
$w(1/2-)\le w(1/2+)$, whereas monotonicity of $v$ gives the reverse
inequality.  Hence no additional jump occurs and $\TV(w_\ell)\le6$.
\end{proof}

\begin{lemma}[Positive interval layer cake]
\label{lem:interval-layer-cake}
Let $w:[0,1]\to[0,\infty)$ be of bounded variation.  There is a finite
nonnegative Borel measure $\nu$ on
\[
  \cI=\{(a,b):0\le a\le b\le1\}
\]
such that
\begin{equation}
  w(z)=\int_{\cI}\1\{a<z<b\}\,d\nu(a,b)
  \label{eq:interval-layer-cake}
\end{equation}
for almost every $z$, and
\begin{equation}
  \nu(\cI)\le\|w\|_\infty+\frac12\TV(w).
  \label{eq:interval-mass}
\end{equation}
\end{lemma}

\begin{proof}[Proof of \Cref{lem:interval-layer-cake}]
For $t\ge0$, let $E_t=\{z:w(z)>t\}$.  For almost every $t$, the
one-dimensional finite-perimeter set $E_t$ is, up to a null set, a finite
union of disjoint intervals
\[
  E_t=\bigcup_{j=1}^{m(t)}(a_{t,j},b_{t,j}).
\]
The one-dimensional coarea formula gives
\[
  \TV(w)=\int_0^\infty\operatorname{Per}(E_t)\,dt.
\]
A subset of $[0,1]$ with $m$ interval components satisfies
$m\le1+\operatorname{Per}(E_t)/2$.  Define $\nu$ by
\[
  \int_{\cI}\Psi(a,b)\,d\nu(a,b)
  =\int_0^{\|w\|_\infty}
    \sum_{j=1}^{m(t)}\Psi(a_{t,j},b_{t,j})\,dt
\]
for every nonnegative Borel $\Psi$.  The scalar layer-cake identity yields
\[
  \int_{\cI}\1\{a<z<b\}\,d\nu(a,b)
  =\int_0^{\|w\|_\infty}\1\{w(z)>t\}\,dt=w(z)
\]
for almost every $z$.  Finally,
\begin{align*}
  \nu(\cI)
  &=\int_0^{\|w\|_\infty}m(t)\,dt\\
  &\le\|w\|_\infty
    +\frac12\int_0^\infty\operatorname{Per}(E_t)\,dt,
\end{align*}
which is \eqref{eq:interval-mass}.
\end{proof}

\begin{proof}[Proof of \Cref{prop:clipped-decomposition-front}]
\Cref{lem:clipped-primitive} proves that each $\ell_{a,b}$ is proper.
Apply \Cref{lem:interval-layer-cake} to the curvature density $w_\ell$ from
\Cref{lem:curvature-bv}.  There is a nonnegative measure $\nu_\ell$ with
\[
  w_\ell(z)=\int_{\cI}\1\{a<z<b\}\,d\nu_\ell(a,b)
\]
and
\[
  \nu_\ell(\cI)
  \le2+\frac12\cdot6=5.
\]
Define
\[
  \widetilde F(p)
  =F_\ell(0)+F_\ell'(0)p
   +\int_{\cI}F_{a,b}(p)\,d\nu_\ell(a,b).
\]
Then $\widetilde F''=w_\ell=F_\ell''$ almost everywhere, and
$\widetilde F(0)=F_\ell(0)$,
$\widetilde F'(0)=F_\ell'(0)$.  Absolute continuity therefore implies
$\widetilde F=F_\ell$.

The proper loss generated by $F_\ell$ is
\[
  \ell(p,y)=-F_\ell(p)-(y-p)F_\ell'(p).
\]
Substituting the representation of $F_\ell$, the affine potential contributes
only an outcome-dependent affine term, while each $F_{a,b}$ contributes
$\ell_{a,b}$.  This proves \eqref{eq:clipped-decomposition-front}.
\end{proof}

\begin{proposition}[Clipped-ReLU cover]
\label{prop:clip-cover}
For all $a,b,a',b'$ and $z\in[0,1]$,
\[
  |\varphi_{a,b}(z)-\varphi_{a',b'}(z)|
  \le|a-a'|+|b-b'|.
\]
Moreover,
\[
  \sup_{p,y}
  |\ell_{a,b}(p,y)-\ell_{a',b'}(p,y)|
  \le2(|a-a'|+|b-b'|).
\]
Consequently, for a universal constant $C_1$, the primitive parameter
triangle $\cI$ admits a $\rho$-net of at most $(C_1/\rho)^2$ points, uniform
both in $\varphi_{a,b}$ and in $\ell_{a,b}$.
\end{proposition}

\begin{proof}[Proof of \Cref{prop:clip-cover}]
The map $a\mapsto(z-a)_+$ is $1$-Lipschitz, and similarly for $b$.  Grid
$a$ and $b$ at spacing proportional to $\rho$; the loss-cover bound follows
from the displayed Lipschitz bound on $\varphi_{a,b}$ and the fact that
$F_{a,b}(p)=\int_0^p\varphi_{a,b}(u)\,du$.
\end{proof}

\begin{proof}[Proof of \Cref{thm:offline-convex-lipschitz}]
Take a $\rho$-net of the primitive parameter triangle $\cI$.  Every
$\ell_{a,b}$ is a $1$-Lipschitz proper loss: by \eqref{eq:clipped-loss},
$\ell_{a,b}(p,0)=-F_{a,b}(p)+p\varphi_{a,b}(p)$ and
$\ell_{a,b}(p,1)=-F_{a,b}(p)-(1-p)\varphi_{a,b}(p)$, and differentiating
using $F_{a,b}'=\varphi_{a,b}$ and $\varphi_{a,b}'=\1\{a<z<b\}$
(a.e.)~gives partial-loss derivatives $p\1\{a<p<b\}$ and
$-(1-p)\1\{a<p<b\}$, both bounded by $1$ in absolute value.  By
\Cref{prop:clip-cover}, the net has cardinality at most $(C_1/\rho)^2$.

\emph{Online statement.}  Apply \Cref{thm:finite-lipschitz-family} with
$L=1$ to this finite primitive family, giving, with probability at least
$1-\delta$,
\[
  \sup_{(a,b)\text{ in the net}}\SwapReg_T(\ell_{a,b},\cH)
  =\widetilde O\!\left(T^{1/3}(\log|\cH|)^{2/3}+\log\frac1\delta\right).
\]
The online loss-cover transfer in \Cref{prop:loss-cover-transfer} extends
this bound from the finite net to the full triangle $\cI$ at an additive
cost of $2T\rho$ swap regret.  Choosing $\rho\asymp T^{-2/3}$ makes this cost,
and the resulting growth of the net cardinality, lower order.  Since the
term $c+dy$ cancels from regret, and swap regret is linear in the loss,
integrate the resulting uniform primitive bound against the nonnegative
measure of \Cref{prop:clipped-decomposition-front}: for every swap rule
$\phi$ and every $\ell\in\cL_{\mathrm{cvx}}$ with representing measure
$\nu_\ell$,
\[
  \SwapReg_T(\ell,\phi)
  =\int_{\cI}\SwapReg_T(\ell_{a,b},\phi)\,d\nu_\ell(a,b)
  \le\nu_\ell(\cI)\sup_{(a,b)\in\cI}\SwapReg_T(\ell_{a,b},\cH),
\]
uniformly in $\phi$, and $\nu_\ell(\cI)\le5$.  This proves the online rate.

\emph{Offline statement.}  Apply \Cref{thm:finite-lipschitz-family} with
$L=1$ to the same finite primitive net, giving
\[
  \sup_{(a,b)\text{ in the net}}\SwapAgn_\cD(\ell_{a,b},\cH)
  =\widetilde O\!\left(
    \left(\frac{\log|\cH|}m\right)^{2/3}+\frac{\log(1/\delta)}m
  \right).
\]
The offline loss-cover transfer in \Cref{prop:loss-cover-transfer} extends
this bound from the finite net to the full triangle $\cI$ at an additive
cost of $2\rho$.  Choosing $\rho=m^{-1}$ makes this approximation lower
order.  As above, integrate the resulting uniform primitive bound against
the nonnegative measure of \Cref{prop:clipped-decomposition-front}, whose
total mass is at most five, to obtain the displayed offline rate.
\end{proof}

\section{Beyond Proper Losses}
\label{sec:improper}

We now evaluate the learner after best-response postprocessing but leave the
benchmark unchanged.  This asymmetry is essential: the learner incurs the loss
of a best response to its probability forecast, while the comparator incurs
the original loss of its chosen action.  We give an exact reduction from this
benchmark to ordinary swap-agnostic learning of a single proper loss.
\Cref{lem:canonical-properization} constructs, for any bounded binary loss, a
\emph{canonical properization}: a proper loss with redundant messages whose
truthful message exactly recovers the learner's postprocessed loss and whose
comparator messages dominate every original action.  The price of this exact
reduction is that the induced proper loss uses a message space with
redundant messages, and the relevant comparator complexity is that of the
loss slopes after composition with $\cH$.
\Cref{thm:fixed-arbitrary-loss-transfer} uses
the construction to transfer swap-regret guarantees for the canonical proper
loss back to the original loss, and \Cref{thm:uniform-arbitrary-loss-transfer}
extends the transfer to families of losses.

For a loss on an action space $\cA$ and a class
$\cH\subseteq\{h:\cX\to\cA\}$, we write
\begin{equation}
\begin{aligned}
  \SwapReg_T^{\mathrm{br}}(\ell,\cH)
  :=
  \sup_{\phi:\cP\to\cH}
  \sum_{t=1}^T
  \left[
    \ell(k_\ell(p_t),y_t)
    -
    \ell(\phi(p_t)(x_t),y_t)
  \right],
\end{aligned}
\label{eq:best-response-swap-regret}
\end{equation}
where $\cP$ is the set of prediction values used by the learner.  The
distributional analogue, denoted
$\SwapAgn_\cD^{\mathrm{br}}(\ell,\cH)$, is defined in the same way with
expectation under $(X,Y)\sim\cD$ and with the learner playing
$k_\ell(p(X))$.

\begin{definition}[Proper loss with redundant messages]
\label{def:redundant-proper}
Let $\cR$ be a message space and let
$\lambda:\cR\times\{0,1\}\to[0,1]$.  We say that $\lambda$ is proper with
truth embedding $\tau:[0,1]\to\cR$ if, for every $p\in[0,1]$ and every
$r\in\cR$,
\[
  \E_{Y\sim\Ber(p)}[\lambda(\tau(p),Y)]
  \le
  \E_{Y\sim\Ber(p)}[\lambda(r,Y)].
\]
Ordinary proper binary losses correspond to $\cR=[0,1]$ and $\tau(p)=p$.
\end{definition}

Let $\ell:\cA\times\{0,1\}\to[0,1]$ be a bounded binary loss on an action
space $\cA$.  Write
\begin{align}
  L_p^\ell(a)
  &=(1-p)\ell(a,0)+p\ell(a,1),
  \label{eq:arbitrary-conditional-risk}\\
  \Delta_\ell(a)
  &=\ell(a,1)-\ell(a,0),
  \label{eq:arbitrary-slope}\\
  s_\ell(a)
  &=-\Delta_\ell(a).
  \label{eq:negative-slope}
\end{align}
Assume that a best response is attained for every $p$:
\begin{equation}
  k_\ell(p)\in\argmin_{a\in\cA}L_p^\ell(a).
  \label{eq:attained-best-response}
\end{equation}

\begin{lemma}[Canonical properization of a bounded binary loss]
\label{lem:canonical-properization}
Let $\ell:\cA\times\{0,1\}\to[0,1]$ satisfy
\eqref{eq:attained-best-response}.  Define its negative Bayes envelope
\begin{equation}
  F_\ell(p)=-\min_{a\in\cA}L_p^\ell(a).
  \label{eq:negative-bayes-envelope}
\end{equation}
Then
\begin{equation}
  F_\ell(p)
  =
  \sup_{a\in\cA}\{p\,s_\ell(a)-\ell(a,0)\}.
  \label{eq:envelope-affine-representation}
\end{equation}

Let
\[
  F_\ell^*(s)=\sup_{q\in[0,1]}\{qs-F_\ell(q)\},
  \qquad
  \cS_\ell=\{s_\ell(a):a\in\cA\},
\]
and, for $s\in\cS_\ell$, let
\[
  Q_\ell(s)=\argmax_{q\in[0,1]}\{qs-F_\ell(q)\}.
\]
The set $Q_\ell(s)$ is nonempty.  Define
\begin{equation}
  \cR_\ell=\{(q,s):s\in\cS_\ell,\ q\in Q_\ell(s)\},
  \qquad
  \lambda_\ell((q,s),y)=F_\ell^*(s)-sy.
  \label{eq:canonical-proper-loss}
\end{equation}
For every $p$, put
\[
  s_{\ell,p}=s_\ell(k_\ell(p)),
  \qquad
  \tau_\ell(p)=(p,s_{\ell,p}).
\]
Then $\tau_\ell(p)\in\cR_\ell$ and:
\begin{enumerate}[label=(\roman*)]
  \item $\lambda_\ell$ is a $[0,1]$-valued proper loss with truth embedding
  $\tau_\ell$.
  \item The learner's canonical proper loss agrees exactly with the
  best-response loss:
  \begin{equation}
    \lambda_\ell(\tau_\ell(p),y)=\ell(k_\ell(p),y)
    \qquad\forall p,y.
    \label{eq:learner-exact-properization}
  \end{equation}
  \item For every action $a\in\cA$ and every choice
  $q_\ell(a)\in Q_\ell(s_\ell(a))$, define
  \[
    \rho_\ell(a)=(q_\ell(a),s_\ell(a)).
  \]
  Then
  \begin{equation}
    \lambda_\ell(\rho_\ell(a),y)\le\ell(a,y)
    \qquad\forall a,y.
    \label{eq:comparator-domination}
  \end{equation}
\end{enumerate}
\end{lemma}

\begin{proof}[Proof of \Cref{lem:canonical-properization}]
For every action $a$,
\[
  -L_p^\ell(a)
  =-\ell(a,0)-p(\ell(a,1)-\ell(a,0))
  =p\,s_\ell(a)-\ell(a,0).
\]
Taking the supremum over $a$ proves
\eqref{eq:envelope-affine-representation}.  Hence $F_\ell$ is convex.
Since each $s_\ell(a)\in[-1,1]$, the affine functions in
\eqref{eq:envelope-affine-representation} are $1$-Lipschitz, so $F_\ell$ is
continuous on $[0,1]$.  The maximum defining $F_\ell^*(s)$ is therefore
attained.

Fix $p$.  Since $k_\ell(p)$ is a best response, its affine function is active
at $p$:
\[
  F_\ell(p)=p\,s_{\ell,p}-\ell(k_\ell(p),0).
\]
Thus $s_{\ell,p}\in\partial F_\ell(p)$, equivalently
$p\in Q_\ell(s_{\ell,p})$, so $\tau_\ell(p)\in\cR_\ell$.  Fenchel equality
gives
\[
  F_\ell^*(s_{\ell,p})
  =p\,s_{\ell,p}-F_\ell(p)
  =\ell(k_\ell(p),0),
\]
and therefore
\[
  \lambda_\ell(\tau_\ell(p),y)
  =\ell(k_\ell(p),0)+y\Delta_\ell(k_\ell(p))
  =\ell(k_\ell(p),y).
\]

For properness, let $(q,s)\in\cR_\ell$.  Fenchel--Young gives
\[
  \E_{Y\sim\Ber(p)}[\lambda_\ell((q,s),Y)]
  =F_\ell^*(s)-ps\ge -F_\ell(p),
\]
while the truthful message achieves equality by the previous paragraph.

It remains to check boundedness and comparator domination.  If
$s=s_\ell(a)$, then \eqref{eq:envelope-affine-representation} implies
\[
  F_\ell(q)\ge qs-\ell(a,0)\qquad\forall q,
\]
so $F_\ell^*(s)\le\ell(a,0)\le1$.  Similarly,
$F_\ell^*(s)-s\le\ell(a,1)\le1$.  Evaluating the conjugate at
$q=0$ and $q=1$ gives
\[
  F_\ell^*(s)\ge -F_\ell(0)=\min_{a'}\ell(a',0)\ge0,
  \qquad
  F_\ell^*(s)-s\ge -F_\ell(1)=\min_{a'}\ell(a',1)\ge0.
\]
Thus $\lambda_\ell$ is $[0,1]$-valued, and the two upper bounds are exactly
\eqref{eq:comparator-domination}.
\end{proof}

\begin{example}[Canonical properization of the $0$--$1$ loss]
\label{ex:canonical-properization-01}
Let $\cA=\{0,1\}$ and $\ell(a,y)=\1\{a\ne y\}$.  Then $L_p(0)=p$,
$L_p(1)=1-p$, so $k_\ell(p)=\1\{p>1/2\}$ and
$F_\ell(p)=-\min\{p,1-p\}$.  Since $s_\ell(0)=-1$ and $s_\ell(1)=1$,
$\cS_\ell=\{-1,1\}$, and
\[
  F_\ell^*(-1)=0,\quad Q_\ell(-1)=\left[0,\tfrac12\right],
  \qquad
  F_\ell^*(1)=1,\quad Q_\ell(1)=\left[\tfrac12,1\right].
\]
The canonical proper loss is $\lambda_\ell((q,-1),y)=y$ and
$\lambda_\ell((q,1),y)=1-y$; it does not depend on $q$ at all.  This is
exactly the redundancy named in \Cref{def:redundant-proper}: every
$q\in[0,\frac12]$ paired with $s=-1$ represents the same underlying decision
$a=0$, and every $q\in[\frac12,1]$ paired with $s=1$ represents $a=1$.  One
checks directly that $\lambda_\ell(\tau_\ell(p),y)=\ell(k_\ell(p),y)$ for
every $p,y$, and that $\rho_\ell(0)=(q,-1)$ and $\rho_\ell(1)=(q,1)$, for any
$q$ in the corresponding interval, satisfy
$\lambda_\ell(\rho_\ell(a),y)=\ell(a,y)$ with equality, since here $\cA$
already coincides with the extreme actions of the decision problem.
\end{example}

For a class $\cH\subseteq\{h:\cX\to\cA\}$, define its canonical
properization by
\[
  \rho_\ell\circ\cH
  =
  \{x\mapsto\rho_\ell(h(x)):h\in\cH\}.
\]
For a proper loss with truth embedding $\tau$, write
\[
  \SwapReg_T(\lambda,\cG;\tau)
  =
  \sup_{\psi:\cP\to\cG}
  \sum_{t=1}^T
  \left[
    \lambda(\tau(p_t),y_t)
    -
    \lambda(\psi(p_t)(x_t),y_t)
  \right].
\]

\begin{theorem}[Fixed bounded-loss transfer]
\label{thm:fixed-arbitrary-loss-transfer}
Let $\ell:\cA\times\{0,1\}\to[0,1]$ be fixed and suppose that best
responses are attained.  Let $(\lambda_\ell,\tau_\ell,\rho_\ell)$ be its
canonical properization from \Cref{lem:canonical-properization}.  Then, for
every transcript,
\begin{equation}
  \SwapReg_T^{\mathrm{br}}(\ell,\cH)
  \le
  \SwapReg_T(
    \lambda_\ell,
    \rho_\ell\circ\cH;
    \tau_\ell
  ).
  \label{eq:fixed-loss-transfer}
\end{equation}
Consequently, any swap-regret bound for the single proper loss
$\lambda_\ell$ and transformed comparator class $\rho_\ell\circ\cH$ gives
the same bound for the original loss $\ell$, provided the learner plays
$k_\ell(p)$ after forecasting $p$.
\end{theorem}

\begin{proof}[Proof of \Cref{thm:fixed-arbitrary-loss-transfer}]
Fix a swap rule $\phi:\cP\to\cH$ and define
$\widetilde\phi(p)=\rho_\ell\circ\phi(p)$.  For every round,
\Cref{lem:canonical-properization} gives
\[
  \ell(k_\ell(p_t),y_t)-\ell(\phi(p_t)(x_t),y_t)
  \le
  \lambda_\ell(\tau_\ell(p_t),y_t)
  -
  \lambda_\ell(\widetilde\phi(p_t)(x_t),y_t).
\]
Summing and taking the supremum over $\phi$ proves the claim.
\end{proof}

\begin{corollary}[Finite-class arbitrary losses from \Cref{thm:fixed-arbitrary-loss-transfer}]
\label{cor:fixed-arbitrary-finite-class}
If $|\cH|=M<\infty$, then
$|\rho_\ell\circ\cH|\le M$.  Hence any finite-class proper-loss guarantee for
$\lambda_\ell$ transfers with the same logarithmic dependence on $M$ to
$\SwapReg_T^{\mathrm{br}}(\ell,\cH)$.
\end{corollary}

\begin{proof}[Proof of \Cref{cor:fixed-arbitrary-finite-class}]
The map $h\mapsto\rho_\ell\circ h$ is deterministic, so it cannot increase
cardinality.  Apply \Cref{thm:fixed-arbitrary-loss-transfer}.
\end{proof}

The transformed comparator class is governed by the scalar slope class
\begin{equation}
  \Delta_\ell\circ\cH
  =
  \{x\mapsto\Delta_\ell(h(x)):h\in\cH\}\subseteq[-1,1]^\cX.
  \label{eq:delta-composed-H}
\end{equation}
Indeed, $\lambda_\ell((q,s),y)=F_\ell^*(s)-sy$ depends on the message
$(q,s)\in\cR_\ell$ only through $s$, so the swap regret of $\lambda_\ell$
against $\rho_\ell\circ\cH$ equals the swap regret of the rescaled proper
loss $\bar\lambda_\ell(p,y):=F_\ell^*(2p-1)-(2p-1)y$, $p\in[0,1]$, against the
comparator class $\tfrac12(1-\Delta_\ell\circ\cH)\subseteq[0,1]^\cX$: this
relabels every message by its $s$-coordinate rescaled into $[0,1]$ and
changes no loss value.  This identification is what lets a
bounded-proper-loss guarantee for scalar $[0,1]$-valued comparators, such as
\Cref{thm:bounded-proper-rates}, act directly on the slope class in place of
the redundant message space $\cR_\ell$.

\begin{theorem}[Uniform transfer from arbitrary losses]
\label{thm:uniform-arbitrary-loss-transfer}
Let $\cL$ be a family of $[0,1]$-valued binary losses on the same action
space $\cA$, each with attained best responses.  For each $\ell\in\cL$, let
$(\lambda_\ell,\tau_\ell,\rho_\ell)$ be its canonical properization.  Suppose
a single probability-valued forecasting algorithm produces a transcript for
which, on an event of probability at least $1-\delta$,
\begin{equation}
  \SwapReg_T(
    \lambda_\ell,
    \rho_\ell\circ\cH;
    \tau_\ell
  )
  \le R_T(\ell,\cH,\delta)
  \label{eq:simultaneous-proper-assumption}
\end{equation}
simultaneously for every $\ell\in\cL$.  Then, on the same event,
\begin{equation}
  \sup_{\ell\in\cL}
  \SwapReg_T^{\mathrm{br}}(\ell,\cH)
  \le
  \sup_{\ell\in\cL}R_T(\ell,\cH,\delta).
  \label{eq:simultaneous-arbitrary-conclusion}
\end{equation}

The relevant complexity is therefore the induced slope class
\begin{equation}
  \cS_{\cL,\cH}
  =
  \{x\mapsto-\Delta_\ell(h(x)):\ell\in\cL,\ h\in\cH\},
  \label{eq:family-slope-class}
\end{equation}
viewed, via the reparametrization above, as a single comparator class shared
by every $\ell\in\cL$.
\end{theorem}

\begin{proof}[Proof of \Cref{thm:uniform-arbitrary-loss-transfer}]
Apply \Cref{thm:fixed-arbitrary-loss-transfer} to each $\ell\in\cL$ and use
the simultaneous event \eqref{eq:simultaneous-proper-assumption}.
\end{proof}

Establishing \eqref{eq:simultaneous-proper-assumption} for a whole family
$\cL$ from a single run requires more than $|\cS_{\cL,\cH}|<\infty$: the
algorithm underlying \Cref{thm:bounded-proper-rates} needs a finite range for
its comparator's predicted values, fixed \emph{before} the online interaction
begins, exactly as in the finite-range hypothesis of that theorem
(\Cref{cor:v-family-full-range}).  A bound on how many distinct behaviors
$\cS_{\cL,\cH}$ can exhibit on a sample observed only after the fact does not
supply such a range, since the representative slope values it counts are
determined by the realized transcript and so cannot be fixed in advance of
it.  We instead assume the range directly.

\begin{corollary}[Finite arbitrary-loss family with a shared finite slope
range]
\label{cor:finite-family-arbitrary-loss}
Suppose $|\cL|=K$, $|\cH|=M$, and suppose further that, fixed before the
online interaction, there is a finite set $\Gamma_\cL\subseteq[-1,1]$ with
$-\Delta_\ell(h(x))\in\Gamma_\cL$ for every $\ell\in\cL$, $h\in\cH$, and
$x\in\cX$ -- the slope-class analogue of the finite-range assumption
$\Gamma_\cH$ in \Cref{thm:bounded-proper-rates}.  Write $J=|\Gamma_\cL|$.
Then $|\cS_{\cL,\cH}|\le KM$, and, on one event of probability at least
$1-\delta$, simultaneously for every $\ell\in\cL$,
\[
  \SwapReg_T(
    \lambda_\ell,
    \rho_\ell\circ\cH;
    \tau_\ell
  )
  =
  \widetilde O\!\left(
    \sqrt{T\left(\log(KM)+\log J+\log\frac1\delta\right)}
  \right).
\]
Consequently, by \Cref{thm:uniform-arbitrary-loss-transfer},
\[
  \sup_{\ell\in\cL}
  \SwapReg_T^{\mathrm{br}}(\ell,\cH)
  =
  \widetilde O\!\left(
    \sqrt{T\left(\log(KM)+\log J+\log\frac1\delta\right)}
  \right).
\]
\end{corollary}

\begin{proof}[Proof of \Cref{cor:finite-family-arbitrary-loss}]
The map $(\ell,h)\mapsto-\Delta_\ell(h(\cdot))$ is deterministic on the $KM$
pairs, so $|\cS_{\cL,\cH}|\le KM$.  By assumption, every element of
$\cS_{\cL,\cH}$, viewed after the reparametrization above as a $[0,1]$-valued
comparator, takes values in the rescaled finite set
$\tfrac12(1+\Gamma_\cL)\subseteq[0,1]$ of size $J$.  Running
\Cref{thm:bounded-proper-rates} with comparator class $\cS_{\cL,\cH}$
(of size at most $KM$, in place of $\cH$) and finite range
$\tfrac12(1+\Gamma_\cL)$ (of size $J$, in place of $\Gamma_\cH$) therefore
gives, on one event of probability at least $1-\delta$,
\[
  \sup_{\ell':|\ell'|\le1}
  \SwapReg_T(\ell',\cS_{\cL,\cH})
  =
  \widetilde O\!\left(
    \sqrt{T\left(\log(KM)+\log J+\log\frac1\delta\right)}
  \right).
\]
For each $\ell\in\cL$, restricting the supremum defining
$\SwapReg_T(\bar\lambda_\ell,\cS_{\cL,\cH})$ to swap rules valued in the
sub-collection $\{-\Delta_\ell(h(\cdot)):h\in\cH\}\subseteq\cS_{\cL,\cH}$ can
only decrease it, so the display above bounds
$\SwapReg_T(\lambda_\ell,\rho_\ell\circ\cH;\tau_\ell)$ for every $\ell\in\cL$
simultaneously, on the same event.
\end{proof}

If the action space $\cA$ itself is finite, $\Gamma_\cL$ is automatic: each
$\Delta_\ell$ has domain $\cA$, so $J\le K|\cA|$ with no separate assumption
needed.  This does not make the hypothesis vacuous, only automatic --
\Cref{rem:all-bounded-losses-complexity} below shows that $J$ can already be
as large as $|\cA|$ for a single hypothesis, so finiteness of $\Gamma_\cL$ is
not the same as smallness of $J$.

\begin{remark}[The family of all bounded losses can have maximal complexity]
\label{rem:all-bounded-losses-complexity}
The finite-range assumption $\Gamma_\cL$ in
\Cref{cor:finite-family-arbitrary-loss} cannot be dropped.  For example, let
$\cX=\cA$ be finite and let $\cH=\{h_{\mathrm{id}}\}$ with
$h_{\mathrm{id}}(x)=x$.  For every function $f:\cA\to[-1,1]$, define
\[
  \ell_f(a,0)=\frac{1-f(a)}2,
  \qquad
  \ell_f(a,1)=\frac{1+f(a)}2.
\]
Then $\Delta_{\ell_f}(a)=f(a)$, so
$\{\Delta_{\ell_f}\circ h_{\mathrm{id}}:f:\cA\to[-1,1]\}$ is the class of all
$[-1,1]$-valued functions on $\cX$.  Without a shared finite range for these
slopes, the induced slope class is already unrestricted, and the unrestricted
family of all bounded losses cannot inherit a nontrivial simultaneous rate
depending only on the complexity of $\cH$.
\end{remark}

\section{Discussion and Open Problems}

Several questions are left open by this paper.  The online algorithm here is
deliberately finite-class, and extending its guarantee to infinite classes --
say, through sequential covers, offset sequential Rademacher complexity, or a
second-order online weak learner -- remains open; a set of representatives
chosen only after seeing the transcript is not enough, since the guarantee
must hold uniformly over swap rules chosen after the fact.  The algorithm
also maintains one weight for every bucket, test, sign, and learning rate, so
making it oracle-efficient -- for instance by combining sparse approachability
with a learner that directly optimizes the underlying offset objective,
$\eta\sum_t(y_t-p_t)g_t-4\eta^2\sum_tg_t^2$, since an ordinary first-order
agnostic learner loses the second-order term needed here -- is a natural next
step.  Finally, the proper-loss identity of \Cref{lem:proper-identity} has
natural analogues for multiclass and vector-valued outcomes, through
subgradients of concave Bayes risks, and combining high-dimensional
approachability with these local curvature bounds to get rates that depend
optimally on the outcome dimension is an appealing direction.

In summary, this paper shows that a fixed proper loss's own geometry is
enough to make swap-agnostic learning almost as easy as ordinary agnostic
learning: rewarding a predictor whenever a proposed swap disagrees with it a
lot, rather than penalizing disagreement at a uniform rate, turns the
comparator's own activity into the very quantity that controls it.  An
approachability-based algorithm produces this guarantee online with high
probability, a standard online-to-batch argument carries it over to a
randomized offline predictor, and finite covers then extend the offline
guarantee from a single loss to broad loss families.  This one idea gives
fast rates for Lipschitz and smooth proper losses, extends to bounded and
convex Lipschitz families by writing them as positive combinations of
simpler basis losses, and transfers exactly to postprocessed losses beyond
the proper case through canonical properization.

\end{document}